\documentclass{article}

\usepackage[utf8]{inputenc}
\usepackage[margin=1in]{geometry}
\usepackage{fullpage,amsfonts}
\usepackage{amsmath}
\usepackage{graphicx}
\usepackage{subcaption}
\usepackage{amsthm}
\usepackage{amssymb}
\usepackage{mathtools}
\usepackage{bbm}
\usepackage{stackrel}
\usepackage{cite}
\usepackage{bm}
\usepackage{tikz}
\usetikzlibrary{fit,positioning,arrows,automata}
\usepackage{algorithm}
\usepackage{algpseudocode}

\numberwithin{equation}{section}

\definecolor{forest}{rgb}{0,0.5,0.0}
\usepackage{url}
\usepackage[colorlinks=true]{hyperref}
\hypersetup{citecolor=forest}      

\theoremstyle{plain}
\newtheorem{theorem}{Theorem}[section]

\newtheorem{lemma}[theorem]{Lemma}

\newtheorem{remark}[theorem]{Remark}

\usepackage{sidecap}
\usepackage{enumitem}
\usepackage{lipsum}
\usepackage{graphicx}
\usepackage{subcaption}
\usepackage{epstopdf}
\ifpdf
  \DeclareGraphicsExtensions{.eps,.pdf,.png,.jpg}
\else
  \DeclareGraphicsExtensions{.eps}
\fi

\usepackage[capitalize]{cleveref}

\crefname{figure}{Fig.}{Fig.}
\crefname{table}{Table}{Table}
\crefname{algorithm}{Alg.}{Algs.}
\crefformat{section}{\S#2#1#3}
\crefmultiformat{section}{\S\S#2#1#3}{ and~#2#1#3}{, #2#1#3}{, and~#2#1#3}
\AtBeginDocument{%
\crefformat{subsection}{\S#2#1#3}
\Crefformat{subsection}{\S#2#1#3}
\crefmultiformat{subsection}{\S\S#2#1#3}{ and~#2#1#3}{, #2#1#3}{, and~#2#1#3}
\crefformat{subsubsection}{\S#2#1#3}
\Crefformat{subsubsection}{\S#2#1#3}
}

\crefname{hypothesis}{Hypothesis}{Hypotheses}

\usepackage{amsopn}

\makeatletter
\newcommand*{\addFileDependency}[1]{
  \typeout{(#1)}
  \@addtofilelist{#1}
  \IfFileExists{#1}{}{\typeout{No file #1.}}
}
\makeatother

\newcommand{\email}[1]{\protect\href{mailto:#1}{#1}}

\usepackage{bm} 

\usepackage{bbold}

\usepackage{amsmath,amsbsy,amsgen,amscd,amsfonts,amssymb} 

\usepackage{setspace}

\usepackage[font=small,labelfont=bf]{caption}

\usepackage{algorithm}
\usepackage{algpseudocode}
\usepackage{comment}

\usepackage{tikz}
\usetikzlibrary{fit,positioning,arrows,automata,calc}

\usepackage{mathtools}
\mathtoolsset{centercolon}  

\renewcommand{\phi}{\varphi}

\renewcommand{\triangleq}{:=}

\providecommand{\mathbbm}{\mathbb} 

\newcommand{\R}{\mathbbm{R}}

\renewcommand{\L}{\mathcal{L}}

\definecolor{mygreen}{rgb}{0.1,0.75,0.2}

\newcommand{\Expect}{\operatorname{\mathbb{E}}}

\newcommand{\Nc}{\mathcal{N}}

\newcommand{\dd}{\text{d}}

\newcommand{\oneton}[1]{\{#1\}_{n=1}^N}

\newcommand*{\B}[1]{\ifmmode\bm{#1}\else\textbf{#1}\fi}
\newcommand{\const}{{\text{const}}}

\newcommand{\EnKF}{{\footnotesize \text{EnKF}}}
\newcommand{\PF}{{\footnotesize \text{PF}}}

\newcommand{\NN}{{\footnotesize \text{NN}}}
\newcommand{\train}{{\footnotesize \text{train}}}
\newcommand{\test}{{\footnotesize \text{test}}}
\newcommand{\EMEnKF}{{\footnotesize \text{EM-EnKF}}}

\newcommand{\LhEnKF}{\L_{\EnKF}}

\newcommand{\LEnKF}{\L_{\EnKF}}
\newcommand{\LPF}{\L_{\PF}}

\newcommand{\LEMEnKF}{\L_{\EMEnKF}}

\newcommand{\LhPF}{\L_{\PF}}

\DeclareMathOperator*{\argmax}{arg\,max}

\newcommand{\E}{\mathbb{E}}
\newcommand{\ctable}[2]{\begin{tabular}[c]{@{}c@{}} #1\\ #2\end{tabular}}
\newcommand{\rmsef}{\text{RMSE-f}}
\newcommand{\rmsea}{\text{RMSE-a}}
\usepackage[percent]{overpic}
\newcommand{\bftab}{\fontseries{b}\selectfont}
\newcommand{\tn}[1]{{#1}_t^n}
\newcommand{\tnf}[1]{ \widehat{#1}_t^{\, n}}
\newcommand{\onf}[1]{\widehat{#1}_1^{\,n}}
\newcommand{\tmn}[1]{{#1}_{t-1}^n}

\newcommand{\tpnf}[1]{\widehat{#1}_{t+1}^{\,n}}
\newcommand{\ton}[1]{{#1}_t^{1:N}}
\newcommand{\tonf}[1]{\widehat{#1}_t^{\,1:N}}

\newcommand{\tf}[1]{\widehat{#1}_t}
\newcommand{\incircle}[1]{\text{\textcircled{\scriptsize #1}}}%
\newcommand{\convincircle}[1]{\xrightarrow[\incircle{#1}]{1/2}}
\newcommand{\iidsim}{\overset{\text{i.i.d.}}{\sim}}
\newcommand{\nt}{\nabla_\theta}
\newcommand{\pt}{\partial_\theta}
\newcommand{\na}{\nabla_\alpha}
\newcommand{\nb}{\nabla_\beta}
\newcommand{\mf}[1]{\mathsf{\bf{#1}}}

\newcommand{\Ss}{\mathbb{S}}
\newcommand{\conv}{\xrightarrow{1/2}}
\newcommand{\Hquad}{\hspace{0.5em}}

\usepackage[normalem]{ulem}

\usepackage{graphicx}

\makeatletter
\protected\def\tikz@nonactivecolon{\ifmmode\mathrel{\mathop\ordinarycolon}\else:\fi} 
\makeatother

\counterwithin{table}{section}
\counterwithin{algorithm}{section}

\title{Auto-differentiable Ensemble Kalman Filters}
\author{Yuming Chen\thanks{University of Chicago, Chicago, IL (\email{ymchen@uchicago.edu, sanzalonso@uchicago.edu, willett@uchicago.edu)}}
\and Daniel Sanz-Alonso\footnotemark[1]
\and Rebecca Willett\footnotemark[1]}
\date{}

\begin{document}
\maketitle

\begin{abstract}
Data assimilation is concerned with sequentially estimating a temporally-evolving state. This task, which arises in a wide range of scientific and engineering applications, is particularly challenging when the state is high-dimensional and the state-space dynamics are unknown. This paper introduces a machine learning framework for learning  dynamical systems in data assimilation. Our auto-differentiable ensemble Kalman filters (AD-EnKFs) blend ensemble Kalman filters for  state recovery with machine learning tools for learning the dynamics. In doing so, AD-EnKFs leverage the ability of ensemble Kalman filters to scale to high-dimensional states and the power of automatic differentiation to train high-dimensional surrogate models for the dynamics. Numerical results using the Lorenz-96 model show that AD-EnKFs outperform existing methods that use expectation-maximization or particle filters to merge data assimilation and machine learning. In addition, AD-EnKFs are easy to implement and require minimal tuning.
\end{abstract}

\section{Introduction}
Time series of data arising across geophysical sciences, remote sensing, automatic control, and a variety of other scientific and engineering applications often reflect observations of an underlying dynamical system operating in a latent state-space. Estimating the evolution of this latent  state from data is the central challenge of data assimilation (DA) \cite{jazwinski2007stochastic,evensen2009data,sanzstuarttaeb,law2015data,reich2015probabilistic}. However, in these and other applications, we often lack an accurate model of the underlying dynamics, and the dynamical model needs to be learned from the observations to perform DA. This paper introduces auto-differentiable ensemble Kalman filters (AD-EnKFs), a machine learning (ML) framework for the principled co-learning of states and dynamics. This framework enables learning in three core categories of unknown dynamics: (a) parametric dynamical models with unknown parameter values; (b) fully-unknown dynamics captured using neural network (NN) surrogate models; and (c) inaccurate or partially-known dynamical models that can be improved using NN corrections. AD-EnKFs are designed to scale to high-dimensional states, observations, and NN surrogate models.

In order to describe the main idea behind the AD-EnKF framework, let us introduce briefly the problem of interest. Our setting will be formalized in \cref{sec:formulation} below. Let  $x_{0:T} \triangleq \{x_t\}_{t=0}^T$ be a time-homogeneous \emph{state process} with transition kernel $p_\theta(x_t|x_{t-1})$ parameterized by  a vector $\theta.$ For instance, $\theta$ may contain unknown parameters of a parametric dynamical model or the parameters of a NN surrogate model for the dynamics. Our aim is to learn $\theta$ from observations $y_{1:T}  \triangleq \{y_t\}_{t=1}^T $ of the state, and thereby learn the unknown dynamics and estimate the state process.  
The AD-EnKF framework learns $\theta$ iteratively. Each iteration consists of three steps: (i) use EnKF to compute an estimate $\LhEnKF(\theta)$ of the data log-likelihood $\L(\theta) \triangleq \log p_\theta(y_{1:T});$ (ii) use auto-differentiation (``autodiff'') to compute the gradient $\nt \LhEnKF(\theta);$ and (iii) take a gradient ascent step. Filtered estimates of the state are obtained using the learned dynamics.

The EnKF, reviewed in \cref{sec:logandgrad}, estimates the data log-likelihood using an ensemble of particles. Precisely, given a transition kernel $p_\theta(x_t|x_{t-1})$, the EnKF generates particles $x_{0:T}^{1:N};$ that is, each particle $x_t^n \in \R^{d_x}$ and 
$x_{0:T}^{1:N} \triangleq \{x_t^n\}_{{t = 0,\ldots T \atop n = 1,\ldots,N}}.$
The log-likelihood estimate $\LhEnKF(\theta)$ depends on $\theta$ through these particles and also through the given transition kernel.
Differentiating the map $\theta \mapsto \LhEnKF(\theta)$  in step (ii) of AD-EnKF involves differentiating \emph{both} the map $\theta \mapsto x_{0:T}^{1:N}$ from parameter to EnKF particles and the map  $(\theta,x_{0:T}^{1:N}) \mapsto \LhEnKF(\theta)$ from parameters and EnKF particles to EnKF log-likelihood estimate. {\em A key feature of our approach is that $\theta \mapsto \LhEnKF(\theta)$ can be auto-differentiated using the reparameterization trick (\cite{kingma2013auto} and \cref{sec:grad}) and autodiff capabilities of NN software libraries} such as PyTorch \cite{paszke2019pytorch}, JAX \cite{jax2018github}, and Tensorflow \cite{abadi2016tensorflow}. Automatic differentiation is different from numerical differentiation in that derivatives are computed exactly through compositions of elementary functions whose derivatives are known, as opposed to finite difference approximations that cause discretization errors.

The AD-EnKF framework represents a significant conceptual and methodological departure from existing approaches to blend DA and ML based on the expectation-maximization (EM) framework, see \cref{fig:compGraph} below.  
Specifically, at each iteration, EM methods that build on the EnKF \cite{pulido2018stochastic,brajard2020combining,bocquet2020bayesian} employ a surrogate likelihood $\LEMEnKF \bigl(\theta\, ; x_{0:T}^{1:N} \bigr)$ where the particles $x_{0:T}^{1:N}$ are generated by EnKF and \emph{fixed}.
Importantly, EM methods compute gradients used to learn dynamics by differentiating only through the $\theta$-dependence in $\LEMEnKF$ that does not involve the particles. In particular, in contrast to AD-EnKF, the map $\theta \mapsto x_{0:T}^{1:N}$ from parameter to EnKF particles is \emph{not} differentiated.
Moreover, the performance of EM methods is sensitive to the specific choice of EnKF algorithm in use, and the tuning of additional algorithmic parameters of EM can be challenging \cite{brajard2020combining,bocquet2020bayesian}. Our numerical experiments suggest that, even when optimally tuned, EM methods underperform AD-EnKF in high-dimensional regimes. The outperfomance of AD-EnKF may be explained by the additional gradient information obtained by differentiating the map $\theta \mapsto x_{0:T}^{1:N}$.

\begin{figure}[!htb]
\hspace{-1.15cm}
\begin{subfigure}[t]{.3\linewidth}
\centering
\begin{tikzpicture}[scale = .75, every node/.style={scale=.75}]
\tikzstyle{main}=[circle, minimum size = 5mm, thick, draw =black!80, node distance = 10mm]
\tikzstyle{connect}=[-latex, thick]
\tikzstyle{rconnect}=[-latex, thick, draw=blue]
\tikzstyle{box}=[rectangle, draw=black!100]
  \node[main] (theta) {$\theta$};
  \node[main,fill=black!10] (y) [right=of theta]{$y_{1:T}$};
  \node[main] (u) [below =of theta] {$x_{0:T}^{1:N}$};
  \node[main, minimum size=20mm] (ll) [below =of u] {$\LhEnKF$};
  \path (theta) edge [connect] (u);
  \path (y) edge [connect] (u);
  \path (theta) edge [connect,bend right=60] (ll);
  \path (y) edge [connect, bend left=20] (ll);
  \path (u) edge [connect] (ll);
  \draw[black,thick,dotted] ($(u.south west)+(-0.3,-0.3)$)  rectangle ($(y.north east)+(0.3,0.5)$);
\end{tikzpicture}
\subcaption{AD-EnKF}
\end{subfigure}
\hspace{-0.45cm}
\begin{subfigure}[t]{.3\linewidth}
\centering
\begin{tikzpicture}[scale = 0.75, every node/.style={scale=0.75}]
\tikzstyle{main}=[circle, minimum size = 5mm, thick, draw =black!80, node distance = 10mm]
\tikzstyle{connect}=[-latex, thick]
\tikzstyle{box}=[rectangle, draw=black!100]
  \node[main] (theta) {$\theta$};
  \node[main,fill=black!10] (y) [right=of theta]{$y_{1:T}$};
  \node[main] (u) [below =of theta] {$x_{0:T}^{1:N}$};
  \node[main, minimum size=20mm] (ll) [below =of u] {$\L_{\text{EM-EnKF}}$};
  \path (theta) edge [connect,bend right=60] (ll);
  \path (u) edge [connect] (ll);
  \draw[black,thick,dotted] ($(u.south west)+(-0.3,-0.3)$)  rectangle ($(y.north east)+(0.3,0.5)$);  
  \begin{scope}[transparency group, opacity=0.25]
      \path (theta) edge [connect] (u);
      \path (y) edge [connect] (u);
  \end{scope}
\end{tikzpicture}
\subcaption{EM-EnKF}
\end{subfigure}
\hfill
\begin{minipage}[T]{.42\linewidth}
\vspace{-1.8in}
With AD-EnKF, parameters $\theta$ and observations $y_{1:T}$ are used to generate EnKF particles $x_{0:T}^{1:N}$; the particles together with $\theta$ and $y_{1:T}$ are used to compute the likelihood $\LhEnKF$, and the gradient $\nt \LhEnKF$ explicitly accounts for the map from $\theta$ to the particles $x_{0:T}^{1:N}$. In contrast, with EM-EnKF, the likelihood $\LEMEnKF$ is a function of $\theta$ and {\em fixed} particles $x_{0:T}^{1:N}$ generated by EnKF, so that computing the gradient $\nt \L_{\text{EM-EnKF}}$ does {\em not} account for the map from $\theta$ to the particles $x_{0:T}^{1:N}$. 
\end{minipage}
\caption{
Computational graph of AD-EnKF and EM-EnKF. Dashed squares represent computations performed by the EnKF. Gray arrows in (b) indicate that the construction of $\L_{\text{EM-EnKF}}$ is performed in two steps: (1) obtain $x_{0:T}^{1:N}$ from $\theta$ and $y_{1:T}$ (gray arrows); and (2) use $\theta$ and $x_{0:T}^{1:N}$ (no longer seen as a function of $\theta$) to define $\L_{\text{EM-EnKF}}.$ In contrast, those lines are black in (a), indicating that in AD-EnKF the particles $x_{0:T}^{1:N}$ in $\LhEnKF$ are seen as varying with $\theta$.
} 
\label{fig:compGraph}
\end{figure}

The AD-EnKF framework also represents a methodological shift from existing differentiable particle filters \cite{naesseth2018variational,maddison2017filtering,le2017auto}. Similar to AD-EnKF, these methods rely on autodiff of a map $\theta \mapsto \LhPF(\theta),$ where the log-likelihood estimate $\LhPF(\theta)$ depends on $\theta$ through weighted particles  $(w_{0:T}^{1:N}, x_{0:T}^{1:N})$ obtained by running a particle filter (PF) with transition kernel $p_\theta(x_t|x_{t-1})$. However, the use of PF suffers from two caveats. First, it is not possible to auto-differentiate directly through the PF resampling steps \cite{naesseth2018variational,maddison2017filtering,le2017auto}. Second, while the PF log-likelihood estimates are consistent, their variance can be large, especially in high-dimensional systems. Moreover, their \emph{gradient}, which is the quantity used to perform gradient ascent to learn $\theta$, is not consistent \cite{corenflos2021differentiable}.  

\subsection{Contributions}
This paper seeks to set the foundations and illustrate the capabilities of the AD-EnKF framework through rigorous theory and systematic numerical experiments. Our main contributions are:
\begin{itemize}
\item We develop new theoretical convergence guarantees for the large sample EnKF estimation of log-likelihood gradients in linear-Gaussian settings (\cref{prop:gradient}).
\item We combine  ideas from online training of recurrent networks (specifically, Truncated Backpropagation Through Time -- TBPTT) with the learning of AD-EnKF when the data sequence is long, i.e. $T$ is large.
\item We provide numerical evidence of the superior estimation accuracies of log-likelihoods and gradients afforded by EnKF relative to PF methods in high-dimensional settings. In particular, we illustrate the importance of using localization techniques, developed in the DA literature, for EnKF log-likelihood and gradient estimation, and the corresponding performance boost within AD-EnKF.
\item We conduct a numerical case study of AD-EnKF on the Lorenz-96 model \cite{lorenz1996predictability}, considering parameterized dynamics, fully-unknown dynamics, and correction of an inaccurate model. The importance of the Lorenz-96 model in geophysical applications and for testing the efficacy of filtering algorithms is highlighted, for instance, in \cite{majda2012filtering,law2012evaluating,law2016filter,brajard2020combining}. Our results show that AD-EnKF outperforms existing methods based on EM or differentiable PFs. The improvements are most significant in challenging high-dimensional and partially-observed settings. 
\end{itemize}

\subsection{Related Work}
The EnKF algorithm was developed as a state estimation tool for DA \cite{evensen1994sequential} and is now widely used in  numerical weather prediction and geophysical applications \cite{szunyogh2008local,whitaker2008ensemble}. Recent reviews include \cite{houtekamer2016review, katzfuss2016understanding,roth2017ensemble}. The idea behind the EnKF is to propagate $N$ equally-weighted particles through the dynamics and assimilate new observations using Kalman-type updates computed with empirical moments. When the state dimension $d_x$ is high and the ensemble size $N$ is moderate, traditional Kalman-type methods require $O(d_x^2)$ memory to store full covariance matrices, while storing empirical covariances in EnKFs only requires $O(N d_x)$ memory. 
The use of EnKF for joint learning of state and model parameters by \emph{state augmentation} was introduced in \cite{anderson2001ensemble}, where EnKF is run on an augmented state-space that includes the state and parameters. However, this approach requires one to design a pseudo-dynamic for the parameters which needs careful tuning and can be problematic when certain types of parameters (e.g., error covariance matrices) are involved \cite{stroud2007sequential,delsole2010state} or if the dimension of the parameters is high.
In this paper, we employ EnKFs to approximate the data log-likelihood. The use of EnKF to perform derivative-free maximum likelihood estimation (MLE) is studied in \cite{stroud2010ensemble,pulido2018stochastic}. An empirical comparison of the likelihood computed using the  EnKF and other filtering algorithms is made in \cite{carrassi2017estimating}; see also \cite{hannart2016dada,metref2019estimating}. The paper \cite{drovandi2021ensemble} uses  EnKF likelihood estimates to design a pseudo-marginal Markov chain Monte Carlo (MCMC) method  for Bayesian inference of model parameters. The works \cite{stroud2007sequential,stroud2018bayesian} propose online Bayesian parameter estimation using the likelihood computed from the  EnKF under a certain family of conjugate distributions. However, to the best of our knowledge, there is no prior work on state and parameter estimation that utilizes gradient information of the EnKF likelihood.

The embedding of EnKF and ensemble Kalman smoothers (EnKS) into the EM algorithm for MLE \cite{dempster1977maximum, bishop2006pattern} has been studied in \cite{tandeo2015offline,ueno2014iterative, dreano2017estimating, pulido2018stochastic}, with a special focus on estimation of error covariance matrices.  The expectation step (E-step) is approximated with EnKS under the Monte Carlo EM framework \cite{wei1990monte}. In addition, \cite{brajard2020combining, nguyen2019like} incorporate deep learning techniques in the maximization step (M-step) to train NN surrogate models. The paper \cite{bocquet2020bayesian} proposes Bayesian estimation of model error statistics, in addition to an NN emulator for the dynamics. On the other hand, \cite{ueno2016bayesian,cocucci2021model} consider online EM methods for error covariance estimation with EnKF. Although gradient information is used during the M-step to train the surrogate model \cite{brajard2020combining, nguyen2019like,bocquet2020bayesian},  these methods do not auto-differentiate through the EnKF (see \cref{fig:compGraph}), and accurate approximation of the E-step is hard to achieve with EnKF or EnKS. 

Another popular approach for state and parameter estimation are particle filters (PFs) \cite{gordon1993novel,doucet2009tutorial} that approximate the filtering step by propagating samples with a kernel, reweighing them with importance sampling, and resampling to avoid weight degeneracy. PFs give an unbiased estimate of the data likelihood \cite{del2004feynman,andrieu2010particle}. Based upon this likelihood estimate, a particle MCMC Bayesian parameter estimation method is designed in \cite{andrieu2010particle}. Although PF likelihood estimates are unbiased, they suffer from two important caveats. First, their variance can be large, as they inherit the weight degeneracy of importance sampling in high dimensions \cite{snyder2008obstacles,bocquet2010beyond,agapiou2017importance,sanz2018importance,sanz2021bayesian}. 
Second, while the propagation and reweighing steps of PFs can be  auto-differentiated, the resampling steps involve discrete distributions 
that cannot be 
handled by
the reparameterization trick. 
For this reason, previous differentiable PFs omit autodiff of the resampling step \cite{naesseth2018variational,maddison2017filtering,le2017auto}, introducing a bias.


An alternative to MLE methods is to optimize a lower bound of the data log-likelihood with variational inference (VI) \cite{bishop2006pattern,kingma2013auto,ranganath2014black}. The posterior distribution over the latent states is approximated with a parametric distribution and is jointly optimized with model parameters defining the underlying state-space model. In this direction, variational sequential Monte Carlo (VSMC) methods  \cite{naesseth2018variational,maddison2017filtering,le2017auto} construct the lower bound using a PF algorithm. Moreover, the proposal distribution of the PF is parameterized and jointly optimized with model parameters defining the state-space model. 
Although VSMC methods provide consistent data log-likelihood estimates, they suffer from the same two caveats as likelihood-based PF methods. Other works that build on the VI framework include \cite{krishnan2017structured,rangapuram2018deep,fraccaro2017disentangled}. An important challenge is to obtain suitable parameterizations of the posterior, especially when the state dimension is high. For this reason, a restrictive Gaussian parameterization with a diagonal covariance matrix is often used in practice \cite{krishnan2017structured,fraccaro2017disentangled}.

\subsection*{Outline}
This paper is organized as follows. \Cref{sec:formulation} formalizes our framework and reviews a characterization of the likelihood in terms of normalizing constants arising in sequential filtering. \Cref{sec:logandgrad} overviews EnKF algorithms for filtering and log-likelihood estimation. \Cref{sec:mainalgorithm} contains our main methodological contributions. Numerical experiments on linear-Gaussian and Lorenz-96 models are described in
\cref{sec:experiments}. We close in \cref{sec:conclusions}.

\subsection*{Notation}
We denote by $t\in\{0,1,\dots,T\}$ a discrete time index and by $n \in \{1,\dots,N\}$ a particle index. Time indices will be denoted with subscripts and particles with superscripts, so that $x_t^n$ represents a generic particle at time $t.$ We denote $x_{t_0:t_1}  \triangleq\{x_t\}_{t=t_0}^{t_1}$ and $ x^{n_1:n_2} \triangleq \{x^n\}_{n=n_0}^{n_1}$. The collection $x_{t_0:t_1}^{n_0:n_1}$ is defined in a similar way. The Gaussian density with mean $m$ and covariance $C$ evaluated at $x$ is denoted by $\Nc(x;m,C)$. The corresponding Gaussian distribution is denoted by $\Nc(m, C)$. For square matrices $A$ and $B$, we write $A \succ B$ if $A-B$ is positive definite, and $A \succeq B$ if $A-B$ is positive semi-definite. For $A \succeq 0$, we denote by $A^{1/2}$ the unique matrix $B \succeq 0$ such that $B^2=A$. We denote by $|v|$  the 2-norm of a vector $v$, and by $|A|$ the Frobenius norm of a matrix $A$. 

\section{Problem Formulation}\label{sec:formulation}
Let  $x_{0:T}$ be a time-homogeneous Markov chain of hidden \emph{states} $x_t \in \R^{d_x}$ with transition kernel $p_\theta(x_t|x_{t-1})$ parameterized by  $\theta \in \R^{d_\theta}.$ Let $y_{1:T}$ be observations of the state. We seek to learn the parameter $\theta$ and recover the state process $x_{0:T}$ from the observations $y_{1:T}$. In \cref{ssec:setting}, we formalize our problem setting, emphasizing our main goal of learning unknown dynamical systems for improved DA. \Cref{ssec:sequentialfiltering} describes how the log-likelihood $\L(\theta) = \log p_\theta(y_{1:T})$ can be written in terms of normalizing constants arising from sequential filtering. This idea will be used in  \cref{sec:logandgrad} to obtain EnKF estimates for $\L(\theta)$ and $\nt \L(\theta)$, which are then employed in \cref{sec:mainalgorithm} to learn $\theta$ by gradient ascent. 

\subsection{Setting and Motivation}\label{ssec:setting}
We consider the following state-space model (SSM)
\begin{alignat}{5}
    &\text{(transition)}&& \quad\quad x_t &&= F_\alpha (x_{t-1}) + \xi_t, \quad\quad && \xi_t \sim \Nc(0, Q_\beta), \quad\quad && 1 \le t \le T, \label{eq:main_state}\\
    &\text{(observation)}&& \quad\quad y_t &&= H x_t + \eta_t, \quad\quad && \eta_t \sim \Nc(0, R), \quad\quad && 1\le t \le T,\label{eq:main_obs}\\
    &\text{(initialization)}&& \quad\quad x_0 &&\sim p_0(x_0).\label{eq:main_init}
\end{alignat}
The initial distribution $p_0$ and the matrices $H \in R^{d_y \times d_x}$ and $R\succ 0$ are assumed to be known. Nonlinear observations can be dealt with by augmenting the state. We further assume independence of all random variables $x_0,$ $\xi_{1:T}$ and $\eta_{1:T}.$ Finally, the transition kernel $p_\theta(x_t|x_{t-1}) = \Nc(x_t;F_\alpha(x_{t-1}),Q_\beta),$  parameterized by $\theta \triangleq\{\alpha,\beta\},$ is defined in terms of a deterministic map $F_\alpha$ and Gaussian additive noise. 
This kernel approximates an unknown state transition of the form
\begin{equation}\label{eq:idealizedmodel}  
   x_t = F^* (x_{t-1}) + \xi_t, \quad\quad  \xi_t \sim \Nc(0, Q^*), \quad\quad  1 \le t \le T,
\end{equation}
where $Q^* = 0$ if the true evolution of the state is deterministic. The parameter $\beta$ allows us to estimate the possibly unknown $Q^*.$ We consider three categories of unknown state transition $F^*$, leading to three types of learning problems:
\begin{enumerate}[label=(\alph*)]
    \item  \emph{Parameterized dynamics}: $F^* = F_{\alpha^*}$ is parameterized, but the true parameter $\alpha^*$ is unknown and needs to be estimated. 
    \item \emph{Fully-unknown dynamics}: $F^*$ is fully unknown and $\alpha$ represents the parameters of a NN surrogate model $F_\alpha^\NN$ for $F^*.$ The goal is to find an accurate surrogate model $F_\alpha^\NN$.  
    \item \emph{Model correction}: $F^*$ is unknown, but an inaccurate model $F_{\text{approx}} \approx F^*$ is available. Here $\alpha$ represents the parameters of a NN $G_\alpha^\NN$ used to correct the inaccurate model. The goal is to learn $\alpha$ so that $F_\alpha \triangleq F_{\text{approx}} + G_\alpha^\NN$ approximates $F^*$ accurately.
\end{enumerate}
In some applications, the map $F^*$ may represent the flow between observations of an autonomous differential equation driving the state, i.e.
\begin{equation}\label{eq:main_ODE_flow}
    \frac{\dd x}{\dd s} = f^*(x), \quad \quad F^* : x(s) \mapsto x(s+\Delta_s),
\end{equation}
where $f^*$ is an unknown vector field and $\Delta_s$ is the time between observations. Then, the map $F_\alpha$ in \eqref{eq:main_state} (resp. $F_\alpha^\NN$, $F_\text{approx}$, $G_\alpha^\NN$) will be similarly defined as the $\Delta_s$-flow of a differential equation with vector field $f_\alpha$ (resp. $f_\alpha^\NN$, $f_\text{approx}$, $g_\alpha^\NN$).  
Once $\theta = \{\alpha,\beta\}$ is learned, the state $x_{0:T}$ can be recovered with a filtering algorithm using the transition kernel $p_\theta(x_t|x_{t-1})$. We will illustrate the implementation and performance of AD-EnKF in these three categories of unknown dynamics in \cref{sec:experiments} using the Lorenz-96 model to define the vector field $f^*$. We remark that learning NN surrogate models for the dynamics may be useful even when the true state transition $F^*$ is known, since $F_\alpha^\NN$ may be cheaper to evaluate than $F^*$. 

\subsection{Sequential Filtering and Data Log-likelihood}\label{ssec:sequentialfiltering}
Suppose that $\theta = \{\alpha,\beta\}$ is known. We recall that, for $1 \le t \le T,$ the \emph{filtering distributions} $p_\theta(x_t|y_{1:t})$ of the SSM \eqref{eq:main_state}-\eqref{eq:main_obs}-\eqref{eq:main_init} can be obtained sequentially, alternating between \emph{forecast} and \emph{analysis} steps:
\begin{align}
\text{(forecast)}& \quad\quad p_\theta(x_t|y_{1:t-1}) = \int \Nc(x_t; F_\alpha(x_{t-1}), Q_\beta) p_\theta(x_{t-1}|y_{1:t-1}) \dd x_{t-1}, \label{eq:forecast}\\
\text{(analysis)}& \quad\quad p_\theta(x_t|y_{1:t}) = \frac{1}{Z_t(\theta)} \Nc(y_t; H x_t, R)  p_\theta(x_t|y_{1:t-1}), \label{eq:filtering}
\end{align}
with the convention $p_\theta(\cdot|y_{1:0}) \triangleq p_\theta(\cdot)$.
Here $Z_t(\theta)$ is a normalizing constant which does not depend on $x_t$. It can be easily shown that
\begin{equation}\label{eq:normalizing}
    Z_t(\theta) = p_\theta(y_t|y_{1:t-1})=\int \Nc(y_t; H x_t, R)  p_\theta(x_t|y_{1:t-1}) \dd x_t,
\end{equation}
and therefore the data log-likelihood admits the characterization
\begin{equation}\label{eq:likecharacterization}
\L(\theta) \triangleq  \log p_\theta(y_{1:T}) = \sum_{t=1}^T \log  p_\theta(y_t|y_{1:t-1}) = \sum_{t=1}^T \log  Z_t(\theta).
\end{equation}
Analytical expressions of the filtering distributions $p_\theta(x_t|y_{1:t})$ and the data log-likelihood $\L(\theta)$ are only available for a small class of SSMs, which includes linear-Gaussian and discrete SSMs \cite{kalman1960new,papaspiliopoulos2014optimal}. Outside these special cases, filtering algorithms need to be employed to approximate the filtering distributions, and these algorithms can be leveraged to estimate the log-likelihood.


\section{Ensemble Kalman Filter Estimation of the Log-likelihood and its Gradient}\label{sec:logandgrad}
In this section, we briefly review EnKFs and how they can be used to obtain an estimate $\LEnKF(\theta)$ of the log-likelihood $\L(\theta)$. 
As will be detailed in \cref{sec:mainalgorithm}, the map $\theta \mapsto \LEnKF(\theta)$ can be readily auto-differentiated to compute $\nt \LEnKF(\theta)$, and this gradient can be used to learn the parameter $\theta$. 
\Cref{ssec:enkf} gives background on EnKFs, \cref{ssec:approxlog-likelihood} shows how EnKFs can be used to estimate $\L(\theta)$, and \cref{ssec:convergence} contains novel convergence guarantees for the EnKF estimation of $\L(\theta)$ and $\nt \L(\theta).$

\subsection{Ensemble Kalman Filters}\label{ssec:enkf}
Given $\theta = \{\alpha, \beta\},$ the EnKF algorithm \cite{evensen1994sequential, evensen2009data} sequentially approximates  the filtering distributions $p_\theta(x_t|y_{1:t})$ using $N$ equally-weighted particles $\ton{x}.$ 
At forecast steps, each particle $\tn{x}$ is propagated using the state transition equation \cref{eq:main_state}, while at analysis steps a Kalman-type update is performed for each particle:
\begin{alignat}{3}
\text{(forecast step)}& \quad\quad \tnf{x} = F_\alpha(\tmn{x}) + \tn{\xi}, &&\quad\quad \tn{\xi}\iidsim \Nc(0, Q_\beta), \label{eq:EnKF_forecast}\\
\text{(analysis step)}& \quad\quad \tn{x} = \tnf{x} + \widehat K_t (y_t + \tn{\gamma} - H \tnf{x}), &&\quad \quad \tn{\gamma} \iidsim  \Nc(0, R). \label{eq:EnKF_filtering}
\end{alignat}

Note that the particles $x_{0:T}^{1:N}$ depend on $\theta,$ and \eqref{eq:EnKF_forecast}-\eqref{eq:EnKF_filtering} implicitly define a map $\theta \mapsto x_{0:T}^{1:N}.$
The Kalman gain $\widehat K_t \triangleq \widehat C_t H^{\top} (H \widehat C_t H^{\top} + R)^{-1}$ is defined using the empirical covariance $\widehat{C}_t$ of the forecast ensemble $\ton{\widehat x},$ namely
\begin{equation}\label{eq:empiricalmeancov}
    \widehat C_t = \frac{1}{N-1} \sum_{n=1}^N (\tnf{x} - \widehat m_t) (\tnf{x} - \widehat m_t)^{\top}, \quad \text{where} \quad \widehat m_t = \frac{1}{N} \sum_{n=1}^N \tnf{x}.
\end{equation}
These empirical moments provide a Gaussian approximation to the \emph{forecast distribution} 
\begin{equation}\label{eq:forecastdist}
    p_\theta(x_t|y_{1:t-1}) \approx \Nc(\widehat m_t, \widehat C_t).
\end{equation}
Several implementations of EnKF are available, but for concreteness we only consider the ``perturbed observation'' 
EnKF defined in \eqref{eq:EnKF_forecast}-\eqref{eq:EnKF_filtering}. In the analysis step \eqref{eq:EnKF_filtering}, the observation $y_t$ is perturbed to form $y_t + \tn{\gamma}.$  
This perturbation ensures that in linear-Gaussian models the empirical mean and covariance of $\ton{x}$ converges as $N\rightarrow \infty$ to the mean and covariance of the filtering distribution  \cite{le2009large,law2016deterministic}.

\subsection{Estimation of the Log-Likelihood and its Gradient}\label{ssec:approxlog-likelihood}
Note from \eqref{eq:likecharacterization} that in order to approximate $\L(\theta) = \log p_\theta(y_{1:T})$, it suffices to approximate $p_\theta(y_t|y_{1:t-1})$ for $1\le t \le T$. Now, using \eqref{eq:normalizing} and the EnKF approximation \eqref{eq:forecastdist} to the forecast distribution, we obtain
\begin{equation}
    p_\theta(y_t|y_{1:t-1}) \approx \int \Nc(y_t; H x_t, R) \Nc(x_t; \widehat m_t, \widehat C_t) \dd x_t = \Nc \big(y_t; H \widehat m_t, H \widehat C_t H^{\top} + R\big).
\end{equation}
Therefore, we have the following estimate of the data log-likelihood:
\begin{equation}\label{eq:EnKF_loglike}
   \LhEnKF(\theta) \triangleq \sum_{t=1}^T \log \Nc \big(y_t; H \widehat m_t, H \widehat C_t H^{\top} + R\big) \approx \L(\theta).
\end{equation}
Notice that the forecast empirical moments $\{\widehat m_t, \widehat C_t\}_{t=1}^T$, and hence $\LhEnKF(\theta),$ depend on $\theta$ in two distinct ways. First, each forecast particle $\tnf{x}$ in \eqref{eq:EnKF_forecast} depends on a particle $\tn{x},$ which indirectly depends on $\theta.$ Second, each forecast particle depends on $\theta = \{\alpha,\beta\}$ directly through $F_\alpha$ and $Q_\beta$.
The estimate $\LhEnKF(\theta)$  can be computed online with EnKF. 
The whole procedure is summarized in Algorithm \ref{alg:EnKF}, which implicitly defines a map
$\theta \mapsto \LhEnKF(\theta).$ 
Before discussing the autodiff of this map and learning of the parameter $\theta$ in \cref{sec:mainalgorithm}, we establish the large ensemble convergence of $\LhEnKF(\theta)$ and $\nt \LhEnKF(\theta)$ towards $\L(\theta)$ and $\nt \L(\theta)$ in a linear setting. 

\begin{algorithm}
\caption{Ensemble Kalman Filter and Log-likelihood Estimation}
\label{alg:EnKF}
\begin{algorithmic}[1]
\Statex {\bf Input}:  $\theta = \{\alpha,\beta\}, y_{1:T}, x_0^{1:N}.$ (If $x_0^{1:N}$ is not specified, draw $x_0^n \iidsim  p_0(x_0)$.)
\State  {\bf Initialize} $\LhEnKF(\theta)=0.$ \label{eq:EnKF_init_alg}
\For {$t= 1, \ldots, T$}
\State Set $\tnf{x} = F_\alpha(\tmn{x}) + \tn{\xi}$, where $\tn{\xi}\iidsim \Nc(0, Q_\beta)$. \label{eq:EnKF_forecast_alg} \Comment{Forecast step}
\State Compute $\widehat m_t, \widehat C_t$ by \cref{eq:empiricalmeancov} and set $\widehat K_t = \widehat C_t H^{\top} (H \widehat C_t H^{\top} + R)^{-1}.$ \label{eq:EnKF_mean_cov_alg}
\State  Set $\tn{x} = \tnf{x} + \widehat K_t (y_t + \tn{\gamma} - H \tnf{x})$, where $\tn{\gamma}\iidsim \Nc(0, R)$. \label{eq:EnKF-analysis-alg} \Comment{Analysis step}
\State    Set $\LhEnKF(\theta) \leftarrow \LhEnKF(\theta) + \log \Nc\big( y_t; H \widehat m_t, H \widehat{C}_t H^{\top} + R \big).$ 
\EndFor
\Statex {\bf Output}: EnKF particles $x_{0:T}^{1:N}$. Log-likelihood estimate $\LhEnKF(\theta)$.
\end{algorithmic}
\end{algorithm}

\subsection{Large Sample Convergence: Linear Setting}\label{ssec:convergence}
In this section we consider a linear setting and provide large $N$ convergence results for the log-likelihood estimate $\LhEnKF(\theta)$ and its gradient $\nt\LhEnKF(\theta)$ towards $\L(\theta)$ and $\nt \L(\theta)$ for any given $\theta$, for a fixed data sequence $y_{1:T}$. The mappings $\L$ and $\LhEnKF$ are defined in \cref{eq:likecharacterization} and \cref{eq:EnKF_loglike}, respectively. For notation convenience, we drop $\theta$ in the function argument since the main dependence will be on $N$ in this section. Similar to \cite{le2009large,kwiatkowski2015convergence}, we study $L^p$ convergence for any $p\ge 1$.

\begin{theorem}\label{prop:likelihood}
Assume that the state transition \cref{eq:main_state} is linear, i.e., 
\begin{equation}\label{eq:linear_assump_thm}
x_t=A_\alpha x_{t-1}+\xi_t, \quad\quad \xi_t \sim \Nc(0,Q_
\beta), \quad\quad A_\alpha \in \R^{d_x \times d_x},
\end{equation}
and that the initial distribution $p_0$ is Gaussian. Then, for any $\theta = \{\alpha,\beta\}$ and  
for any $p\ge 1,$
$\LhEnKF$ converges to $\L$ in $L^p$  with rate $1/\sqrt N$, i.e.,
\begin{equation}
    \bigl(\Expect \big| \LhEnKF - \L \big|^p \bigr)^{1/p} \le c N^{-1/2},
\end{equation}
where $c$ does not depend on $N$.
\end{theorem}

 The linearity of the flow $F_\alpha(\cdot)$ is equivalent to the linearity of the vector field $f_\alpha(\cdot)$. Although the convergence of EnKF to the KF in linear settings has been studied in DA \cite{le2009large, law2016deterministic,kwiatkowski2015convergence, del2018stability} and in filtering approaches to inverse problems \cite{schillings2017analysis,chada2020iterative}, there are no existing convergence results for EnKF log-likelihood estimation. An exception is \cite{katzfuss2020ensemble}, where the authors provide a heuristic argument for convergence in the case $T=1$. 
Most of the theoretical analysis of EnKF is based on the \emph{propagation of chaos} statement \cite{mckean1967propagation,sznitman1991topics}: EnKF defines an interacting particle system, where the interaction is through the empirical mean $\widehat m_t$ and covariance matrix $\widehat C_t$ of the forecast ensemble $\ton{\widehat x}$. As $N\rightarrow\infty$, one hopes that these empirical moments can be replaced by their deterministic limits, and that the particles will hence evolve independently. The large $N$ limits of $\widehat m_t, \widehat C_t$ turn out to be the mean and covariance matrix of the KF forecast distribution. We will leave the construction of the {propagation of chaos} statement as well as the proof of \cref{prop:likelihood} to \cref{sec:likeproof-appen}.

Since this paper focuses on gradient based approaches to the learning of  $\theta=\{\alpha,\beta\}$,  it is thus interesting to compare the gradient $\nt \LhEnKF$ to the true gradient $\nt \L$, as $N\rightarrow\infty$, if both of them exist. The intuition is 
that if $\nt \LhEnKF$ is an accurate estimate of $\nt \L$, then one can perform gradient-based optimization over $\LhEnKF$ as if one was directly optimizing over the true log-likelihood $\L$. For the gradient w.r.t.\  $\beta$ to be well-defined, we write $S_\beta = Q_\beta^{1/2}$ 
in the following statement so that $\beta$ does not appear in the stochasticity of the algorithm. This is also known as the ``reparameterization trick,'' which will be discussed later in \cref{sec:grad}.
\begin{theorem}\label{prop:gradient}
Assume that the state transition \cref{eq:main_state} is linear, i.e., 
\begin{equation}
x_t=A_\alpha x_{t-1}+ S_\beta \xi_t, \quad\quad  \xi_t \sim \Nc(0,I_{d_x}), \quad\quad A_\alpha\in \R^{d_x \times d_x},
\end{equation}
and that the initial distribution $p_0$ is Gaussian. Assume the parameterizations $\alpha \mapsto A_\alpha$ and $\beta \mapsto S_\beta$ are differentiable. Then, for any $\theta=\{\alpha,\beta\}$, both $\nt \LhEnKF$ and $\nt \L$ exist and, for any $p \ge 1,$ $\nt \LhEnKF$ converges to $\nt \L$ in $L^p$ with rate $1/\sqrt N$, i.e.,
\begin{equation}
    \bigl(\Expect \big| \nt \LhEnKF - \nt \L \big|^p\bigr)^{1/p} \le c N^{-1/2},
\end{equation}
where $c$ does not depend on $N$.
\end{theorem}
An important observation is that $\theta$ only enters the objective function $\LhEnKF$ through the empirical mean $\widehat m_t$ and covariance matrix $\widehat C_t$ of the forecast ensemble. As $N\rightarrow\infty$, one hopes that these empirical moments can be replaced by their deterministic limits, and gradients based on these empirical moments can be replaced by gradients based on their deterministic limits.
The gradients taken in the limits turn out to be those of the true log-likelihood $\L$. Again, the proof relies on the propagation of chaos statement and is left to \cref{sec:gradproof-appen}.

\section{Auto-differentiable Ensemble Kalman Filters}\label{sec:mainalgorithm}
\sloppypar
This section contains our main methodological contributions. We introduce our AD-EnKF framework in \cref{sec:grad}. We then describe in \cref{sec:TBPTT} how to handle long observation data, i.e., large $T$, using TBPTT. In \cref{sec:localization}, we highlight how various techniques introduced for EnKF in the DA community, e.g., localization and covariance inflation, can be incorporated into our framework. Finally, \cref{ssec:cost} discusses the computational and memory costs.

\subsection{Main Algorithm}\label{sec:grad}
\begin{algorithm}
\caption{Auto-differentiable Ensemble Kalman Filter (AD-EnKF)}
\label{alg:AD-EnKF}
\begin{algorithmic}[1]
\Statex {\bf Input}:  Observations $y_{1:T}$. Learning rate $\eta$. 
\State {\bf Initialize} SSM parameter $\theta^0$ and set $k=0.$
\While {not converging}
\State $x_{0:T}^{1:N} , \LhEnKF(\theta^k) = \textsc{EnsembleKalmanFilter}(\theta^k, y_{1:T}).$ \label{eq:EnKF_main_1}
\Comment{\cref{alg:EnKF}}
\State Compute $\nt \LhEnKF(\theta^k)$ by auto-differentiating the map $\theta^k \mapsto \LEnKF(\theta^k).$
\State Set $\theta^{k+1} = \theta^{k} + \eta \nt \LhEnKF(\theta^k)$ and $k\leftarrow k+1.$
\EndWhile
\Statex {\bf Output}: Learned SSM parameter $\theta^k$  and EnKF particles $x_{0:T}^{1:N}$.
\end{algorithmic}
\end{algorithm}


Our core method is shown in \cref{alg:AD-EnKF}, and our PyTorch implementation is at \url{https://github.com/ymchen0/torchEnKF}.
The gradient of  the map $\theta^k \mapsto \LEnKF(\theta^k)$ can be evaluated 
using autodiff libraries  \cite{paszke2019pytorch,jax2018github,abadi2016tensorflow}. 
More specifically, reverse-mode autodiff can be performed for common matrix operations like matrix multiplication, inverse, and determinant  \cite{giles2008collected}. 
We use the ``reparameterization trick'' \cite{kingma2013auto,rezende2014stochastic} to auto-differentiate through the stochasticity in the EnKF algorithm.
Specifically, in \cref{alg:EnKF} line \ref{eq:EnKF_forecast_alg},  we draw $\tn{\xi}$ from a distribution $\Nc(0,Q_\beta)$ that involves a parameter $\beta$ with respect to which we would like to compute the gradient. For this operation to be compatible with the autodiff, we reparameterize
\begin{equation}\label{eq:reparam}
    \widehat x_t^n =F_\alpha(x_t^n)+\tn{\xi}\quad \tn{\xi} \iidsim \Nc(0,Q_\beta) \quad \Longleftrightarrow \quad \widehat x_t^n =F_\alpha(x_t^n)+ Q_\beta^{1/2} \tn{\xi} \quad \xi_t^n \iidsim \Nc(0, I_{d_x}),
\end{equation}
so that the gradient with respect to $\beta$ admits an  unbiased estimate. In contrast to the EnKF, the resampling step of PFs cannot be readily auto-differentiated \cite{naesseth2018variational,maddison2017filtering,le2017auto}.

\subsection{Truncated Gradients for Long Sequences}\label{sec:TBPTT}
If the sequence length $T$ is large, although $\LhEnKF(\theta)$ and its gradient $\nt \LhEnKF(\theta)$ can be evaluated using the aforementioned techniques, the practical value of  \cref{alg:AD-EnKF} is limited for two reasons. First, computing these quantities requires a full filtering pass of the data, which may be computationally costly. Moreover, for the gradient ascent methods to achieve a good convergence rate, multiple evaluations of gradients are often needed, requiring an equally large number of filtering passes. The second reason is that, like recurrent networks, 
\cref{alg:AD-EnKF}
may suffer from  exploding or vanishing gradients \cite{pascanu2013difficulty} as the derivatives are multiplied together using chain rules in the backpropagation.

Our proposed technique can address both of these issues by borrowing the ideas of TBPTT from the recurrent neural network literature  \cite{williams1995gradient, sutskever2014sequence} and the recursive maximum likelihood method from the hidden Markov models literature \cite{le1997recursive}. The idea is to divide the sequence into subsequences of length $L$. Instead of computing the log-likelihood of the whole sequence and then backpropagating, one computes the log-likelihood of each subsequence and backpropagates within that subsequence. The subsequences are processed sequentially, and the EnKF output of the previous subsequence (i.e., the location of particles) are used as the input to the next subsequence. In this way, one performs $\lceil T/L \rceil$ gradient updates in a \emph{single} filtering pass, and since the gradients are backpropagated across a time span of length at most $L$, gradient explosion/vanishing is more unlikely to happen. This approach is detailed in \cref{alg:AD-EnKF-T}.

\begin{algorithm}
\caption{AD-EnKF with Truncated Backprop (AD-EnKF-T)}
\label{alg:AD-EnKF-T}
\begin{algorithmic}[1]
\Statex {\bf Input}:  Observations $y_{1:T}$. Learning rate $\eta$. Subsequence length $L$.
\State {\bf Initialize} SSM parameter $\theta^0$ and set $k=0.$
\While {not converging}
\State Set $x_0^n \iidsim p_0(x_0)$.
\For {$j=0,\ldots, T/L-1$}
\State Set $t_0=jL$, $t_1=\min \{(j+1)L, T\}$.
\State $x_{t_0:t_1}^{1:N}, \LhEnKF(\theta^k) = \textsc{EnsembleKalmanFilter}(\theta^k, y_{(t_0+1):t_1}, x_{t_0}^{1:N}).$ \Comment{\cref{alg:EnKF}}
\State Set $\theta^{k+1} = \theta^{k} + \eta \nt \LhEnKF(\theta^k)$ and $k \leftarrow k+1.$ 
\EndFor
\EndWhile
\Statex {\bf Output}: Learned SSM parameter $\theta^k$  and EnKF particles $x_{0:T}^{1:N}$.

\end{algorithmic}
\end{algorithm}

\subsection{Localization for High State Dimensions}\label{sec:localization}
In practice, the state often represents a physical quantity that is discretized in spatial coordinates (e.g., numerical solution to a time-evolving PDE), which leads to a high state dimension $d_x$. In order to reduce the computational and memory complexity, EnKF is often run with $N < d_x$. A small ensemble size $N$ causes rank deficiency of the forecast sample covariance $\widehat C_t$, which may cause spurious correlations between spatial coordinates that are far apart. In other words, for $(i,j)$ such that $|i-j|$ is large, the $(i,j)$-th coordinate of $\widehat C_t$ may not be close to 0, although one would expect it to be small since it represents the correlation between spatial locations that are far apart. This problem can be addressed using localization techniques, and we shall focus on \emph{covariance tapering}  \cite{houtekamer1998data}. The idea is to ``taper'' the forecast sample covariance matrix $\widehat C_t$ so that the nonzero spurious correlations are zeroed out. This method is implemented defining a $d_x \times d_x$ matrix $\rho$ with 1's on the diagonal and entries smoothly decaying to $0$ off the diagonal, and  replacing the forecast sample covariance matrix $\widehat C_t$ in \cref{alg:EnKF} by $\rho \circ \widehat C_t$, where $\circ$ denotes the element-wise matrix product. Common choices of $\rho$ were introduced in \cite{gaspari1999construction}. Covariance tapering can be easily adopted within our AD-EnKF framework. 
We find that covariance tapering not only stabilizes the filtering procedure, which had been noted before, e.g., \cite{houtekamer2001sequential,hamill2001distance}, but it also helps to obtain low-variance estimates of the log-likelihood and its gradient --- see the discussion in \cref{sec:linear_loc}. Localization techniques relying on local serial updating of the state \cite{houtekamer2001sequential,ott2004local,sakov2011relation} could also be considered.

Another useful tool for EnKF with $N<d_x$ is \emph{covariance inflation} \cite{anderson1999monte}, which prevents the ensemble from collapsing towards its mean after the analysis update \cite{furrer2007estimation}. In practice, this can be performed by replacing the forecast sample covariance matrix $\widehat C_t$ in \cref{alg:EnKF} by $(1+\zeta) \widehat C_t$, where $\zeta>0$ is a small constant that needs to be tuned. Although not considered in our experiments, covariance inflation can also be easily adopted within our AD-EnKF framework.

\subsection{Computation and Memory Costs}\label{ssec:cost}

Autodifferentiation of the map $\theta^k \mapsto \LhEnKF(\theta^k)$ in \cref{alg:AD-EnKF} does not introduce an extra order of computational cost compared to the evaluation of this map alone. Thus, the computational cost of AD-EnKF is at the same order as that of a standard EnKF. The computation cost of EnKF can be found in, e.g., \cite{roth2017ensemble}. Moreover, AD-EnKF can be parallelized and speeded up with a GPU.

Like a standard EnKF, when no covariance tapering is applied, AD-EnKF has $O(N d_x)$ memory cost since it does not explicitly compute the sample covariance matrix $\widehat C_t$\footnote{Note that 
    $\tf{C} H^\top = \frac{1}{N-1} \sum_{n=1}^N (\tnf{x} - \tf{m})(H\tnf{x}-H\tf{m})^\top$
and
    $H \tf{C} H^\top = \frac{1}{N-1} \sum_{n=1}^N (H\tnf{x} - H \tf{m})(H\tnf{x}-H\tf{m})^\top$, 
which require $O(d_x \max\{N,d_y\})$ and $O(d_y \max\{N, d_y\})$ memory respectively. Both of them are less than $O(d_x^2)$ if $d_y\ll d_x$ and $N \ll d_x$.
}. With covariance tapering, the memory cost is at most $O(\max\{N,r\} d_x)$, where $r$ is the tapering radius, if the tapering matrix $\rho$ is sparse with $O(r d_x)$ nonzero entries. This sparsity condition is satisfied when using common tapering matrices \cite{gaspari1999construction}. In terms of the time dimension,  the memory cost of AD-EnKF can be reduced from $O(T)$ to $O(L)$ with the TBPTT in \cref{sec:TBPTT}. Unlike previous work on EM-based approaches \cite{brajard2020combining,bocquet2020bayesian,pulido2018stochastic}, where the locations of all particles $x_t^{1:N}$ across the whole time span of $T$ need to be stored, AD-EnKF-T only requires to store the particles within a time span of $L$ to perform a gradient step.

If the transition map $F_\alpha$ is defined by the flow map of an ODE with vector field $f_\alpha$, we can use adjoint methods to differentiate  efficiently through $F_\alpha$ in the forecast step \cref{eq:EnKF_forecast}. Use of the adjoint method is facilitated by NeuralODE autodiff libraries   \cite{chen2018neural} that have become an important tool to  learn continuous-time dynamical systems  \cite{ayed2019learning,rubanova2019latent,de2019gru}. Instead of discretizing $F_\alpha$ with a numerical solver applied to $f_\alpha$
and differentiating through solver's steps as in \cite{brajard2020combining,bocquet2020bayesian}, we directly differentiate through $F_\alpha$ by solving an adjoint differential equation, which does not require us to store all intermediate steps from the numerical solver, reducing the memory cost. More details can be found in \cite{chen2018neural}, and the PyTorch package provided by the authors can be incoporated within our AD-EnKF framework with minimal effort.



\section{Numerical Experiments}\label{sec:experiments}
\subsection{Linear-Gaussian Model}\label{sec:linear}
In this section, we focus on parameter estimation in a linear-Gaussian model with a banded structure on model dynamic and model error covariance matrix. This experiment falls into the category of ``parameterized dynamics'' in \cref{ssec:setting}. We first illustrate the convergence results of the log-likelihood estimate $\LhEnKF$ and gradient estimate $\nt\LhEnKF$ presented in \cref{ssec:convergence}, since the true values $\L$ and $\nt \L$ are available in closed form. We also show that the localization techniques described in \cref{sec:localization} lead to a more accurate estimate when the ensemble size is small. 
Finally, we show that having a more accurate estimate, especially for the gradient, improves the parameter estimation.

We  compare the EnKF to PF methods. Similar to the EnKF, the PF also provides an estimate of the log-likelihood and its gradient. Different from \cite{naesseth2018variational,maddison2017filtering,le2017auto}, we adopt the PF with optimal proposal \cite{doucet2009tutorial} as it is implementable for the family of SSMs considered in this paper \cite{doucet2000sequential,sanzstuarttaeb}, and we find it to be more stable than separately training a variational proposal. To compute the log-likelihood gradient for the PF, we follow the same strategy as in \cite{naesseth2018variational,maddison2017filtering,le2017auto} and do not differentiate through the resampling step. The full algorithm,  which we abbreviate as AD-PF, is presented in \cref{sec:pf-appen}.

We consider the following SSM, similar to \cite{xu2007estimation, stroud2018bayesian}
\begin{subequations}
\label{eq:linGauss}
\begin{alignat}{3}
    x_t &= A_\alpha x_{t-1} + \xi_t, \quad\quad && \xi_t \sim \Nc(0, Q_\beta), \quad\quad && 1\le t \le T,\\
    y_t &= H x_t + \eta_t, \quad\quad && \eta_t \sim \Nc(0, 0.5 I_{d_y}), \quad\quad && 1\le t \le T,\\
    x_0 &\sim \Nc (0, 4  I_{d_x}),
\end{alignat}
where 
\begin{equation}
A_\alpha=
\begin{bmatrix}
\alpha_1 & \alpha_2 &          & 0       \\
\alpha_3 & \alpha_1 & \ddots   &         \\
         & \ddots   & \ddots   & \alpha_2\\
0        &          & \alpha_3 & \alpha_1\\
\end{bmatrix},\quad\quad
[Q_\beta]_{i,j} = \beta_1 \exp (-\beta_2 |i-j| ).
\end{equation}
\end{subequations}
Here $[Q_\beta]_{i,j}$ denotes the $(i,j)$-th entry of $Q_\beta$.  Intuitively, $\beta_1$ controls the scale of error, while $\beta_2$ controls how error is correlated across spatial coordinates. We set $\alpha = (\alpha_1, \alpha_2, \alpha_3),$ $\beta=(\beta_1, \beta_2),$ and $\theta=\{\alpha, \beta\}$.

\subsubsection{Estimation Accuracy of $\LhEnKF$ and $\nt \LhEnKF$}\label{sec:linear_converge}

As detailed above, a key idea proposed in this paper is to estimate $\L(\theta)$, $\na \L(\theta)$ and $\nb \L(\theta)$ with quantities $\LhEnKF(\theta)$, $\na \LhEnKF(\theta)$ and $\nb \LhEnKF(\theta)$ obtained by running an EnKF and differentiating through its computations using autodiff. 
Since these estimates will be used by AD-EnKF to perform gradient ascent, it is critical to assess their accuracy. We do so in this section for a range of values of $\theta.$

We first simulate observation data $y_{1:T}$ from the true model with $d_x=d_y\in \{20, 40, 80\}$, $T=10$, $H=I_{d_x}$, $\alpha^*=(0.3, 0.6, 0.1)$ and $\beta^*=(0.5, 1)$. Given data $y_{1:T}$, the true data log-likelihood $\L(\theta)=p_\theta(y_{1:T})$ and gradient $\nabla_\theta \L(\theta)$, which can be decomposed into $\na \L(\theta)$, $\nb \L(\theta)$, can be computed analytically.
We perform $P=50$ EnKF runs, and report a Monte Carlo estimate of the relative $L^2$ errors of the log-likelihood and gradient estimates (see \cref{sec:implement-appen} for their definition)  as the ensemble size $N$ increases.
\cref{fig:linear_converge} shows the results when $\theta$ is evaluated at the true parameters $\{\alpha^*, \beta^*\}$. Intuitively, this $\theta$ is close to optimal since it is the one that generates the data. We also show in \cref{fig:linear_converge_off} in \cref{sec:addfig-appen} the results when $\theta$ is evaluated at a parameter that is not close to optimal: $\alpha=(0.5, 0.5, 0.5), \beta=(1, 0.1)$. Both figures illustrate that the relative $L^2$ estimation errors of the log-likelihood and its gradient computed using EnKF converge to zero at a rate of approximately $N^{-1/2}$. Moreover,  the state dimension $d_x$ has small empirical effect on the convergence rate. On the other hand, those computed using PF have a slower convergence rate or barely converge, especially for the gradient (see the third plot in \cref{fig:linear_converge_off}). We recall that the resampling parts are discarded from the autodiff of PFs, which introduces a bias. Moreover, the empirical convergence rate is slightly slower in higher state dimensions. Comparing the estimation error of EnKF and PF under the same $d_x$ choice, we find that when the number of particles is large ($>500$), EnKF gives a more accurate estimate than PF. However, when the number of particles is small, EnKF is less accurate, but we will show in the next section how the EnKF results can be significantly improved using localization techniques.

\begin{figure}[htbp]
	\centering
	\includegraphics[width=\textwidth]{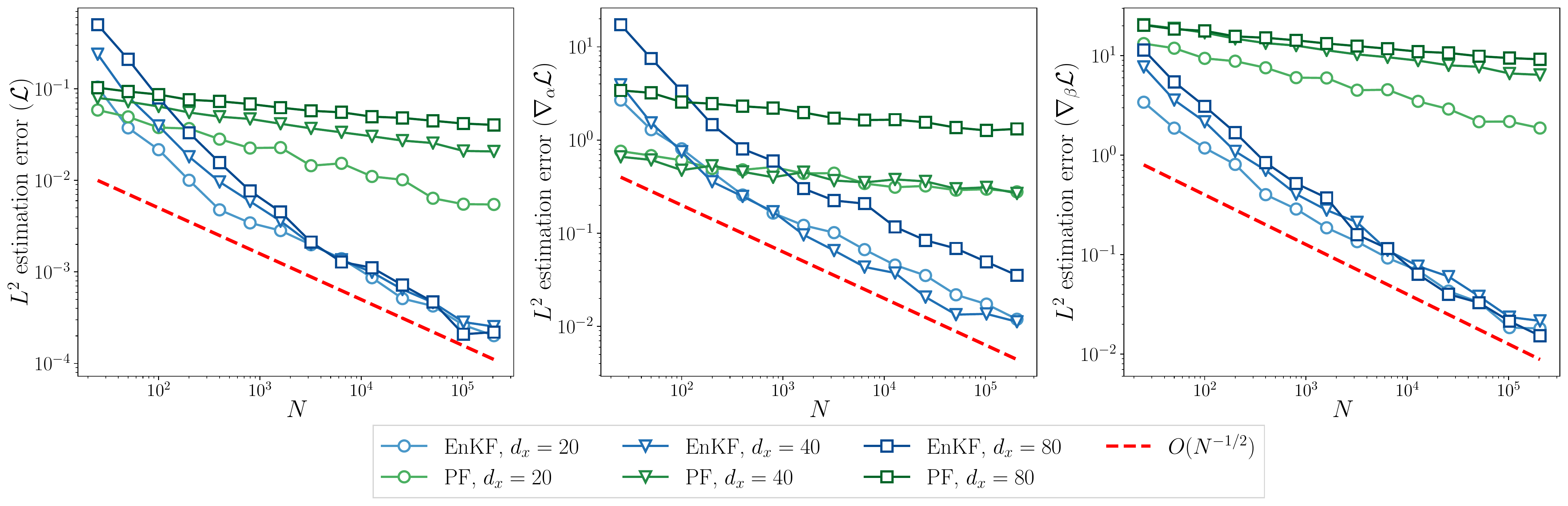}
	\vspace*{-7mm}
	\caption{Relative $L^2$ estimation errors of the log-likelihood (left) and its gradient w.r.t.\  $\alpha$ (middle) and $\beta$ (right), computed using EnKF and PF, as a function of $N$, for the linear-Gaussian model  \eqref{eq:linGauss}. State dimension $d_x\in\{20,40,80\}$. $\theta$ is evaluated at the true parameters $\{\alpha^*, \beta^*\}$. (\cref{sec:linear_converge})}
	 \label{fig:linear_converge}
\end{figure}

\subsubsection{Effect of Localization}\label{sec:linear_loc}
In practice, for computational and memory concerns, the number of particles used for EnKF is typically small ($<100$), and hence it is necessary to get an accurate estimate of log-likelihood and its gradients using a small number of particles. We use the covariance tapering techniques discussed in \cref{sec:localization}, where $\widehat C_t$ is replaced by $\rho \circ \widehat C_t$ in \cref{alg:EnKF}, and $\rho$ is defined using the fifth order piecewise polynomial correlation function of Gasperi and Cohn \cite{gaspari1999construction}. The detailed construction of $\rho$ is left to \cref{sec:implement-appen}, with a hyperparameter $r$ that controls the tapering radius.

\cref{fig:linear_converge_loc} shows the estimation results when the state dimension is set to be $d_x=80$ and $\theta$ is evaluated at $(\alpha^*,\beta^*)$, while different tapering radii $r$ are applied. The plots of EnKF with no tapering and the plots of PF are the same as in \cref{fig:linear_converge}. We find that covariance tapering can reduce the estimation error of the log-likelihood and its gradient when the number of particles is small. Moreover, having a smaller tapering radius leads to a better estimation when the number of particles is small.
As the number of particles grows larger, covariance tapering may worsen the estimation of both log-likelihood and its gradient. This is because the sampling error and spurious correlation that occurs in the sample covariance matrix in EnKF will be overcome by large number of particles, and hence covariance tapering will only act as a modification to the objective function $\LhEnKF$, leading to inconsistent estimates. However, there is no reason for using localization  when one can afford a large number of particles. 
When computational constraints require fewer particles than state dimension, we find that covariance tapering is not only beneficial to the parameter estimation problems but is also beneficial to
learning of the dynamics in high dimensions, as we will show in later sections. Results when $\theta$ is evaluated at parameters that are not optimal ($\alpha=(0.5, 0.5, 0.5)$, $\beta=(1,0.1)$) are shown in \cref{fig:linear_converge_loc_off} in \cref{sec:addfig-appen}, where the beneficial effect of tapering is evident. 

\begin{figure}[htbp]
	\centering
	\includegraphics[width=\textwidth]{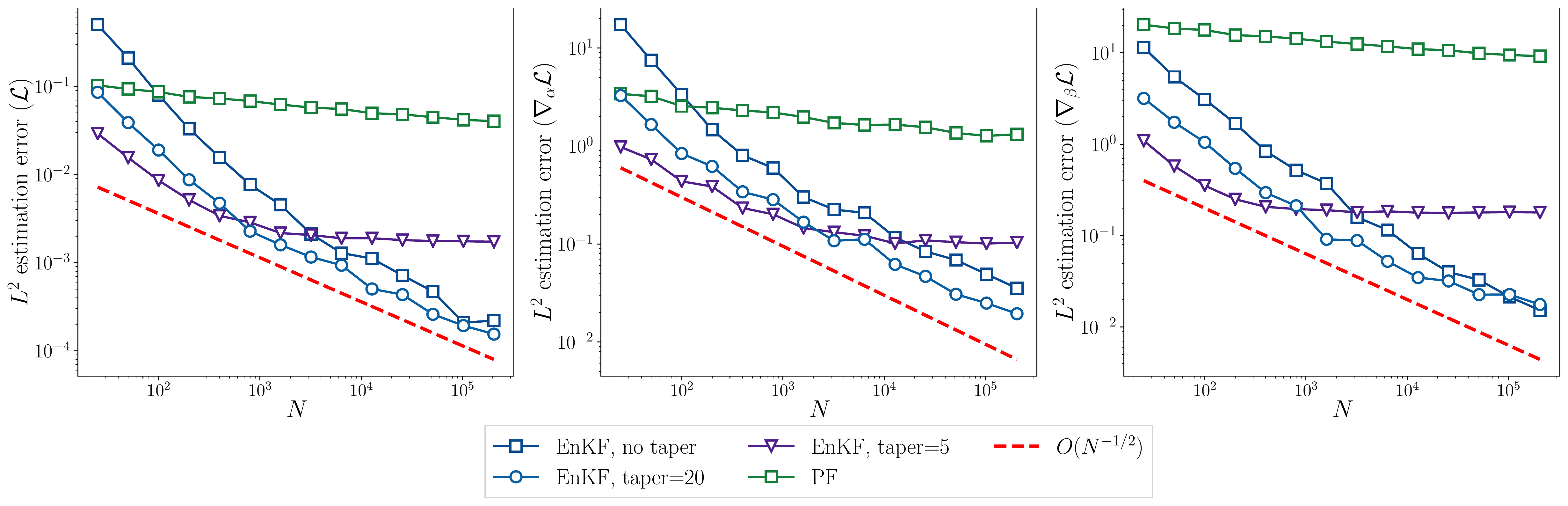}
	\vspace*{-7mm}
	\caption{Relative $L^2$ estimation errors of log-likelihood (left) and its gradient w.r.t.\ $\alpha$ (middle) and $\beta$ (right), computed using EnKF and PF, with different covariance tapering radius applied to EnKF for the linear-Gaussian model  \eqref{eq:linGauss}. State dimension $d_x=80$. $\theta$ is evaluated at the true parameters $\{\alpha^*, \beta^*\}$. (\cref{sec:linear_loc})}
	 \label{fig:linear_converge_loc}
\end{figure}

\subsubsection{Parameter Learning}
\label{sec:paramLin}
Here we illustrate how the estimation accuracy of the  log-likelihood and its gradient,  especially the latter, affect the parameter learning with AD-EnKF. Since our framework relies on gradient-based learning of parameters, intuitively, the less biased the gradient estimate is, the closer our learned parameter will be to the true MLE solution. 

We first consider the setting where the state dimension is set to be $d_x=80$.  We run AD-EnKF for 1000 iterations with gradient ascent under the following choices of ensemble size and tapering radius: (1) $N=1000$ with no tapering; (2) $N=50$ with no tapering; and (3) $N=50$ with tapering radius 5. We also run AD-PF with $N=1000$ particles. Throughout, one ``training iteration'' corresponds to processing once the whole data sequence.
Additional implementation details are available in the appendices. \cref{fig:linear_param_est_1,fig:linear_param_est_2} show a single run of parameter learning under each setting, where we include for reference the MLE obtained by running gradient ascent until convergence with the true gradient $\nt \L$ (denoted with the red dashed line). 
The objective function, i.e., the likelihood estimates $\LhEnKF$ and $\LhPF$ are also plotted as a function of training iterations. 
Results with other choices of state dimension $d_x$ are summarized in \cref{tb:linear_param_est}, where we take the values of $\alpha$ at the final iteration and compute their distance to the true MLE solution. The procedure is repeated 10 times, and the mean and standard deviations are reported. 
The results all show a similar trend: AD-EnKF with $N=1000$ particles performs the best (small errors and small fluctuations) for all settings, while AD-EnKF with $N=50$ particles and covariance tapering performs second best. AD-EnKF with $N=50$ without covariance tapering comes at the third place, and AD-PF method performs the worst, indicating the superiority of AD-EnKF method to the AD-PF method for high-dimensional linear-Gaussian models of the form \eqref{eq:linGauss}
and the utility of localization techniques.
Importantly, the findings here are consistent with the plots in \cref{fig:linear_converge_loc}.
This behavior is in agreement with the intuition that the estimation accuracy of the log-likelihood gradient determines the parameter learning performance. 


\begin{figure}[htbp]
	\centering
	\includegraphics[width=\textwidth]{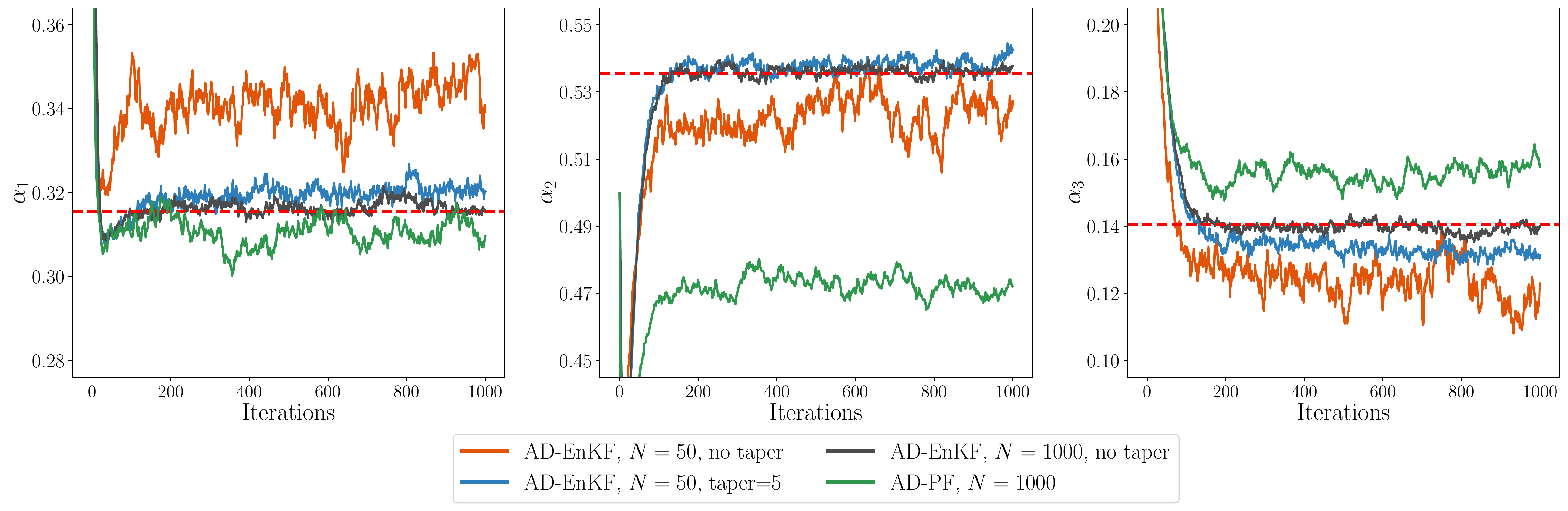}
	\vspace*{-7mm}
	\caption{Learned parameter $\alpha$ as a function of training iterations for the linear-Gaussian model  \eqref{eq:linGauss}. State dimension $d_x=80$. Red dashed lines are the MLE solutions to the true data log-likelihood $\L$.
	(\cref{sec:paramLin})}
	 \label{fig:linear_param_est_1}
\end{figure}
\begin{figure}[htbp]
	\centering
	\includegraphics[width=\textwidth]{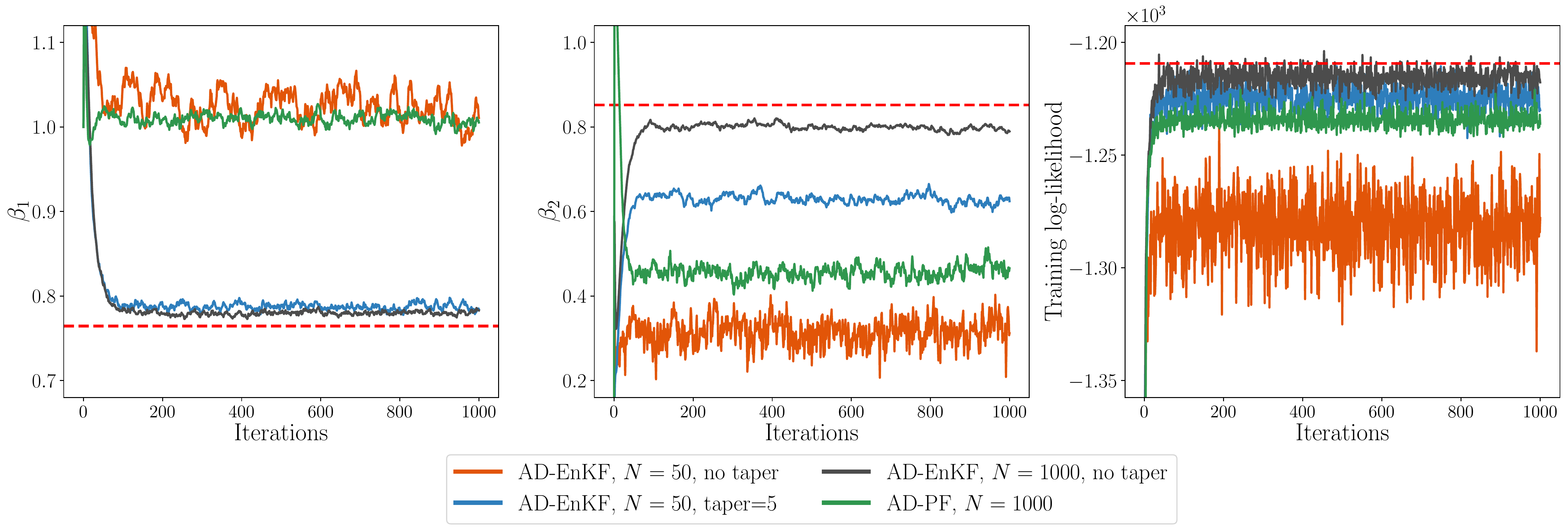}
	\vspace*{-7mm}
	\caption{Learned parameter $\beta$, and  training objective $\LhEnKF$, $\LhPF$ as a function of training iterations for the linear-Gaussian model  \eqref{eq:linGauss}. Red dashed lines are the MLE solutions to the true data log-likelihood $\L$ (left and middle), and the maximum value attained by $\L$ (right). (\cref{sec:paramLin})}
	 \label{fig:linear_param_est_2}
\end{figure}

\begin{table}[htbp]
\centering
{\footnotesize
\begin{tabular}{|c|c|c|c|c|c|c|c|}
\hline
&\ctable{$d_x=20$}{$N=50$} & \ctable{$d_x=20$}{$N=1000$} &
\ctable{$d_x=40$}{$N=50$}&
\ctable{$d_x=40$}{$N=1000$}&
\ctable{$d_x=80$}{$N=50$}&
\ctable{$d_x=80$}{$N=1000$}\\ \hline
{AD-EnKF (no taper)} & 
{1.65 $\pm$0.30}&
{0.07 $\pm$0.06}&
{4.12$\pm$0.73}&
{0.17$\pm$0.09}&
{4.14$\pm$0.67}&
{0.20$\pm$0.14}
\\ \hline
{AD-EnKF(taper=5)}&
{0.53$\pm$0.18}&$-$&
{0.35$\pm$0.27}&$-$&
{1.05}{$\pm$0.38}&$-$
\\ \hline
AD-PF &
{7.75}{$\pm$0.37}&
{3.51}{$\pm$0.35}&
{8.58}{$\pm$0.25}&
{5.59}{$\pm$0.31}&
{9.28}{$\pm$0.49}&
{6.77}{$\pm$0.24}
\\ \hline
\end{tabular}}
\caption{Euclidean distance ($\times 10^{-2}$) from the learned parameter $\alpha$ at the final iteration to the true MLE solution, under varying dimensional settings for the linear-Gaussian model \eqref{eq:linGauss} (\cref{sec:paramLin}).}
\label{tb:linear_param_est}
\end{table}

\subsection{Lorenz-96}
In this section, we illustrate our AD-EnKF framework in the three types of learning problems mentioned in \cref{ssec:setting}: parameterized dynamics, fully-unknown dynamics, and model correction. We will compare our method to AD-PF, as in \cref{sec:linear}. We will also compare our method to the EM-EnKF method implemented in \cite{bocquet2020bayesian,brajard2020combining}, which we abbreviate as EM, and is detailed in \cref{sec:em-appen}.
We emphasize that the gradients computed in the EM are different from the ones computed in AD-EnKF, and in particular do not auto-differentiate through the EnKF. 

The reference Lorenz-96 model \cite{lorenz1996predictability} is defined by \eqref{eq:main_ODE_flow} with vector field
\begin{equation}\label{eq:l96_ode}
    f^{*(i)} (x) = -x^{(i-1)} (x^{(i-2)} - x^{(i+1)}) - x^{(i)} + 8, \quad\quad 0\le i \le d_x-1,
\end{equation}
where $x^{(i)}$ and $f^{*(i)}$ are the $i$-th coordinate of $x$ and component of $f^*$. By convention $x^{(-1)}\triangleq x^{(d_x-1)}, x^{(-2)}\triangleq x^{(d_x-2)}$ and $x^{(d_x)}\triangleq x^{(0)}$. We assume there is no noise in the reference state transition model, i.e., $Q^*=0$. The goal is to recover the reference state transition model with $p_\theta(x_t|x_{t-1})=\Nc(x_t; F_\alpha\big(x_{t-1}), Q_\beta\big)$ from the data $y_{1:T},$ where $F_\alpha$ is the flow map of a vector field $f_\alpha$, and then recover the states $x_{1:T}$. 
The parameterized error covariance $Q_\beta$ in the transition model is assumed to be diagonal, i.e., $Q_\beta=\text{diag}(\beta)$ with $\beta\in \R^{d_x}$. The parameterized vector field $f_\alpha$ is defined differently for the three types of learning problems, as we lay out below.
We quantify performance using the forecast error (RMSE-f), the analysis/filter error (RMSE-a), and the test log-likelihood. These metrics are defined in \cref{sec:implement-appen}.


\subsubsection{Parameterized Dynamics}\label{ssec:l96-param} 
We consider the same setting as in \cite{bocquet2019data}, where 
\begin{equation}\label{eq:l96_param_alpha}
\begin{split}
    f_\alpha^{(i)}(x) = \big[&1, x^{(i-2)}, x^{(i-1)}, x^{(i)}, x^{(i+1)}, x^{(i+2)},\\
    &\big(x^{(i-2)}\big)^2, \big(x^{(i-1)}\big)^2,\big(x^{(i)}\big)^2,\big(x^{(i+1)}\big)^2,\big(x^{(i+2)}\big)^2,\\
    &x^{(i-2)}x^{(i-1)}, x^{(i-1)}x^{(i)}, x^{(i)}x^{(i+1)}, x^{(i+1)}x^{(i+2)},\\
    &x^{(i-2)}x^{(i)}, x^{(i-1)}x^{(i+1)}, x^{(i)}x^{(i+2)} \big]^\top \alpha, \quad\quad 0 \le i \le d_x-1,
\end{split}
\end{equation}
and $\alpha\in\R^{18}$ is interpreted as the coefficients of some ``basis polynomials'' representing the governing equation of the underlying system. The parameterized governing equation of the $i$-th coordinate depends on its $N_1=5$ neighboring coordinates, and the second order polynomials only involve interactions between coordinates that are at most $N_2=2$ indices apart. 
The reference ODE \cref{eq:l96_ode} satisfies $f^*=f_{\alpha^*}$, where $\alpha^*\in\R^{18}$ has nonzero entries 
\begin{equation}\label{eq:alpha-star}
    \alpha^*_0 = 8,\quad \alpha^*_3=-1,\quad \alpha^*_{11}=-1,\quad \alpha^*_{16}=1,
\end{equation}
and zero entries otherwise. Here the dimension of $\theta = \{\alpha,\beta\}$ is $d_\theta=18+d_x$.

We first consider the specific case with $d_x=d_y=40$, $H=I_{40}$. We set $R=I_{40}$ and $x_0 \sim \Nc(0,50 I_{40})$. We generate four sequences of training data with the reference model for $T=300$ with time between consecutive observations $\Delta_s=0.05$. Both flow maps $F^*$ and $F_\alpha$ are integrated using a fourth-order Runge Kutta (RK4) method with step size $\Delta_s^{\text{int}}=0.01$, with adjoint methods implemented for backpropagation through the ODE solver \cite{chen2018neural}.

We use AD-EnKF-T (\cref{alg:AD-EnKF-T})  with $L=20$ and 
covariance tapering \cref{eq:tapering} with radius $r=5$. 
We compare with 
AD-PF-T (see \cref{sec:pf-appen}) with $L=20$
and EM (see \cref{sec:em-appen}).
The implementation details, including the choice of learning rates and other hyperparameters, are discussed in \cref{sec:implement-appen}.

Comparison of the three algorithms is shown in  \cref{fig:l96-param-est}. 
Our AD-EnKF-T  recovers $\alpha^*$ better than the other two approaches. The EM approach converges faster, 
but has a larger error. 
Moreover, EM tends to converge to a higher level of learned model error $\sigma_\beta$ (defined in \cref{eq:diagnosed-error}), while our AD-EnKF-T shows a consistent drop of learned error level.
Note that 
$Q_\beta$ in the learned transition kernel acts like covariance inflation, 
which is discussed in \cref{sec:localization}, but is ``learned'' to be adaptive to the training data rather than manually tuned; therefore, having a nonzero error level $\sigma_\beta$ may still be helpful. The plot of the log-likelihood estimate during training indicates that AD-EnKF-T searches for parameters with a higher log-likelihood than the EM approach,
which is not surprising as AD-EnKF-T directly optimizes $\LhEnKF$, while EM does so by alternatively optimizing a surrogate objective. Also, the large discrepancy between the optimized $\LhEnKF$ and $\LhPF$ objective  may be due to $\LhPF$ being a worse estimate for the true log-likelihood $\L$ than that of $\LhEnKF$. Note that PFs may not be suitable for high-dimensional systems like the Lorenz-96 model. Even with knowledge of the true reference model and a large number of particles, the PF is not able to capture the filtering distribution well due to the high dimensionality --- see, e.g., Figure 5 of \cite{bocquet2010beyond}. 


\begin{figure}
\begin{minipage}[c]{0.73\textwidth}
  \centering   
  \includegraphics[width=0.95\textwidth]{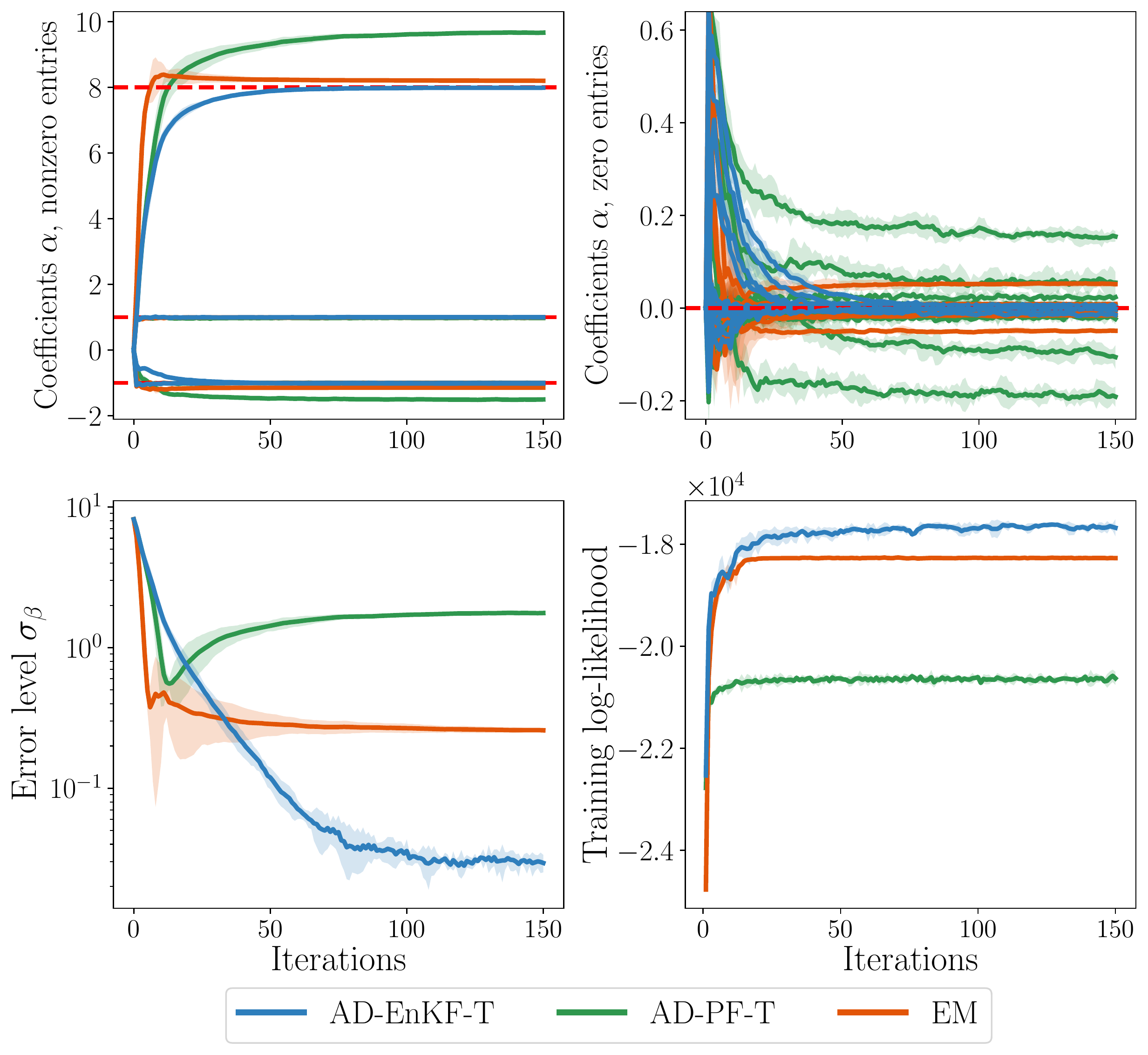}
\vspace*{-3mm}
\end{minipage}
\begin{minipage}[c]{0.25\textwidth}
\caption{\label{fig:l96-param-est} Learning parameterized dynamics of Lorenz-96 \eqref{eq:l96_param_alpha}, with $d_x=40$ and $H=I_{40}$. Learned value of the 18 coefficients of $\alpha$ (upper left for nonzero entries and upper right for zero entries, where the truth $\alpha^*$ is plotted in red dashed lines), averaged diagnosed error level $\sigma_\beta$ \cref{eq:diagnosed-error}
(lower left) and log-likelihood $\LhEnKF$/$\LhPF$ during training (lower right), as a function of training iterations. Throughout, shaded area corresponds to $\pm 2$ std over 5 repeated runs. (\cref{ssec:l96-param})
}
\end{minipage}
\end{figure}

We also consider varying the state dimension $d_x$ and observation model $H$.
(The parameterization 
in \cref{eq:l96_param_alpha}
is valid for any choice of $d_x$.) We measure the Euclidean distance between the value of learned $\alpha$ at the final training iteration (at convergence) to  $\alpha^*$. The training procedure is repeated 5 times and the results are shown in \cref{tb:l96_param_est}. We vary $d_x\in\{10, 20, 40, 80\}$ and consider two settings for $H$: fully observed at all coordinates, i.e., $H=I_{d_x}$, and partially observed at every two out of three coordinates \cite{sanz2015long}, i.e., $H=[e_1, e_2, e_4, e_5, e_7, \cdots]^{\top}$, where $\{e_i\}_{i=1}^{d_x}$ is the standard basis for $\R^{d_x}$. The number of particles used for all algorithms is fixed at $N=50$, and covariance tapering \cref{eq:tapering} with radius $r=5$ is applied to the EnKF. For both AD-EnKF-T and AD-PF-T, 
$L=20$. We find that AD-EnKF-T is able to consistently recover $\alpha^*$ regardless of the choice of $d_x$ and $H$, and is able to perform well in the important case where $N<d_x$, with an accuracy that is orders of magnitude better than the other two approaches. The EM approach is able to recover $\alpha^*$ consistently in fully observed settings, but with a lower accuracy. In partially-observed settings, EM does not converge to the same value in repeated runs, possibly due to the existence of multiple local maxima. AD-PF-T is able to converge consistently in fully observed settings but with the lowest accuracy, and runs into filter divergence issues in partially-observed settings, so that the training process is not able to complete. Moreover, we observe that the error of AD-PF-T tends to grow with the state dimension $d_x$, while the two approaches based on EnKF do not deteriorate when increasing the state dimension. This is further evidence that EnKF is superior in high-dimensional settings.

\begin{table}[htbp]
\resizebox{\columnwidth}{!}{%
\centering
{\footnotesize
\begin{tabular}{|c|c|c|c|c|c|c|c|}
\hline
     & \begin{tabular}[c]{@{}c@{}}$d_x=10$ \\ (full)\end{tabular} & \begin{tabular}[c]{@{}c@{}}$d_x=20$ \\ (full)\end{tabular} & \begin{tabular}[c]{@{}c@{}}$d_x=20$ \\ (partial)\end{tabular} & \begin{tabular}[c]{@{}c@{}}$d_x=40$\\  (full)\end{tabular} & \begin{tabular}[c]{@{}c@{}}$d_x=40$ \\ (partial)\end{tabular} & \begin{tabular}[c]{@{}c@{}}$d_x=80$ \\ (full)\end{tabular} & \begin{tabular}[c]{@{}c@{}}$d_x=80$ \\ (partial)\end{tabular} \\ \hline
EM       & {0.308}{$\pm$ 0.026} & {0.289}{$\pm$ 0.0114} & {2.28}{$\pm$ 4.92} & {0.268}{$\pm$ 0.0103} &
{7.754}{$\pm$ 8.057} & {0.231}{$\pm$ 0.0209} & 
{7.382}{$\pm$ 4.812}\\ \hline
AD-PF-T    & {0.262}{$\pm$ 0.020} & {0.711}{$\pm$ 0.0291} & 
$-$                       & {1.557}{$\pm$ 0.0422} & 
$-$                       & {2.079}{$\pm$ 0.0275} &
$-$ \\ \hline
AD-EnKF-T & \bftab {0.217}{$\pm$ 0.027} & \bftab {0.0325}{$\pm$ 0.0128} & \bftab {0.0835}{$\pm$ 0.0189} & \bftab {0.0283}{$\pm$ 0.0022} & \bftab {0.0930}{$\pm$ 0.0098} & \bftab {0.0540}{$\pm$ 0.0065} & \bftab {0.0813}{$\pm$ 0.0083}\\ \hline
\end{tabular}%
}}
\caption{Lorenz-96, learning parameterized dynamics with varying $d_x$ and observation models. The table shows recovery of learned $\alpha^*$ for each algorithm at the final training iteration, in terms of its distance to the truth $\alpha^*$ \cref{eq:alpha-star}. ``Full'' corresponds to full observations, i.e., $H=I_{d_x}$. ``Partial'' corresponds to observing 2 out of 3 coordinates, i.e., $H=[e_1, e_2, e_4, e_5, e_7, \cdots]^{\top}$. The ``-'' indicates that training cannot be completed due to filter divergence. (\cref{ssec:l96-param})
}
\label{tb:l96_param_est}
\end{table}

\subsubsection{Fully Unknown Dynamics}\label{ssec:l96-unknown}
We assume no knowledge of the reference vector field $f^*$, and we approximate it by a neural network surrogate $f_\alpha^\NN: \R^{d_x} \rightarrow \R^{d_x},$
where here $\alpha$ represents the NN weights. 
The structure of the NN is similar to the one in \cite{brajard2020combining} and is detailed in \cref{sec:implement-appen}. The number of parameters combined for $\alpha$ and $\beta$ is $d_\theta=9317$. The experimental results are compared to the model correction results, and hence are postponed to \cref{ssec:l96-correction}.

\subsubsection{Model Correction}\label{ssec:l96-correction}
We assume $f^*$ is unknown, but that an inaccurate model $f_\text{approx}$ is available. We make use of the parametric form \cref{eq:l96_param_alpha}, and define $f_\text{approx}$ via a perturbation $\widetilde{\alpha}$ of the true parameter $\alpha^*$:
\begin{equation}
    f_\text{approx} \triangleq f_{\widetilde{\alpha}}, \quad\quad \text{where }\widetilde \alpha_i \sim
    \begin{cases}
     \Nc(\alpha_i^*, 1), \quad & \text{if } i=0,\\
     \Nc(\alpha_i^*, 0.1), \quad &\text{if }  i\in\{1,\dots,5\},\\
     \Nc(\alpha_i^*, 0.01), \quad &\text{if }  i\in\{6,\dots,17\}.
    \end{cases}
\end{equation}
The coefficients of a higher order polynomial have a smaller amount of perturbation. $\widetilde\alpha$ is \emph{fixed} throughout the learning procedure. We approximate the residual $f^*-f_{\text{approx}}$ by a NN $g_\alpha^\NN$, 
where $\alpha$ represents the weights, and $g_\alpha^\NN$ has the same structure and the same number of parameters as in the fully unknown 
setting. The goal is to learn $\alpha$ so that $f_\alpha\triangleq f_\text{approx}+g_\alpha^\NN$ approximates $f^*$.

We set  $d_x=40$ and consider two settings for $H$: fully observed with $H=I_{40}$, $d_y=40$, and partially observed at every two out of three coordinates with $d_y=27$ (see \cref{ssec:l96-param}). Eight data sequences are generated with the reference model for training and four for testing, each with length $T=1200$. Other experimental settings are the same as in \cref{ssec:l96-param}.

For the setting where training data is fully observed, we compare AD-EnKF-T with AD-PF-T and the EM approach. The results are plotted in \cref{fig:l96_learn}. The number of particles used for all algorithms is fixed at $N=50$, and covariance tapering \cref{eq:tapering} with radius $r=5$ is applied to EnKF. The subsequence length for both AD-EnKF-T and AD-PF-T is chosen to be $L=20$. We find that, whether $f^*$ is fully known or an inaccurate model is available, AD-EnKF-T is able to learn the reference vector field $f^*$ well, with the smallest forecast RMSE among all methods. Applying a filtering algorithm to the learned model, we find that the states recovered by the AD-EnKF-T algorithm at the final iteration have the lowest error (filter RMSE) among all methods, indicating that AD-EnKF-T also has the ability to learn unknown states well. Moreover, the filter RMSE of AD-EnKF-T is close to the one computed using a filtering algorithm \emph{with}  known $f^*$ and $Q^*$. The test log-likelihood $\LhEnKF$ of the model learned by AD-EnKF-T is close to the one evaluated with the reference model. We also find that having an inaccurate model $f_\text{approx}$ is beneficial to the learning of AD-EnKF-T. 
The performance metrics are boosted compared to the ones with a fully unknown model. 
EM has worse results, where we find that the forecast RMSE does not consistently drop in the training procedure and the states are not accurately recovered. This might be because the smoothing distribution used by EM cannot be approximated accurately. AD-PF-T has the worst performance, possibly because PF fails in high dimensions.

\begin{figure}[htbp]
	\centering
	\includegraphics[width=\textwidth]{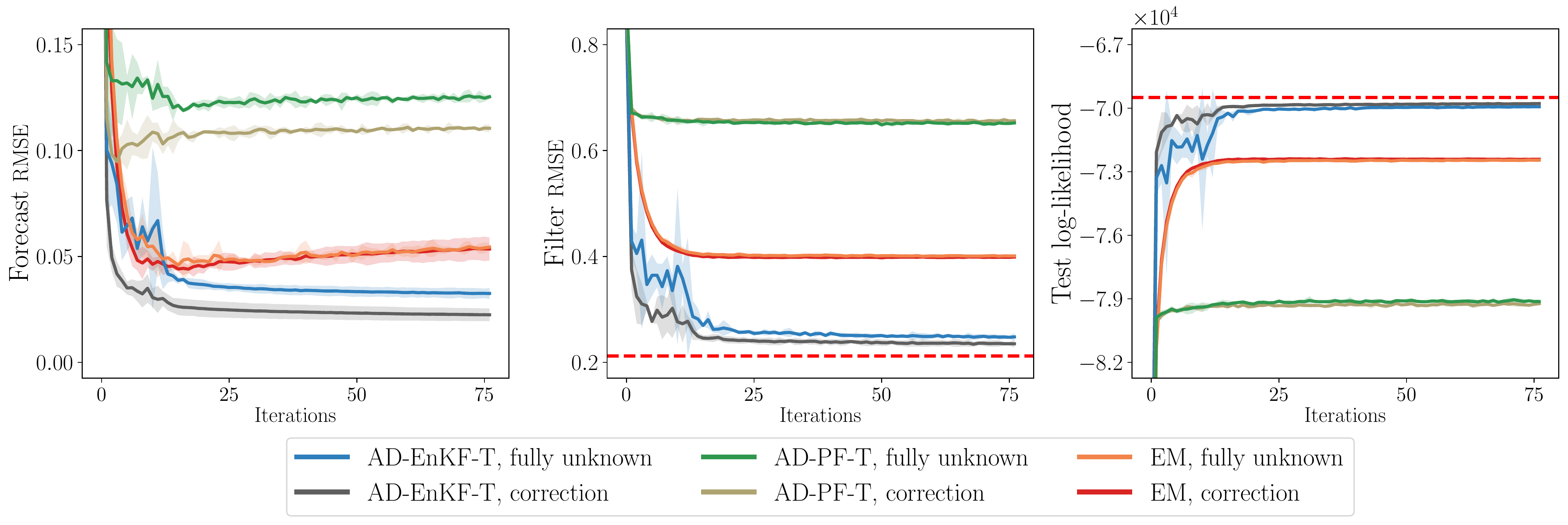}
    \vspace*{-7mm}
    \caption{Learning the Lorenz-96 model from fully unknown dynamics (\cref{ssec:l96-unknown}) v.s.\ model correction (\cref{ssec:l96-correction}), with full observations ($H=I_{d_x}$). All performance metrics are evaluated after each training iteration. Red dashed lines correspond to metric values obtained with the reference model $f^*$ and $Q^*$. }
	 \label{fig:l96_learn}
\end{figure}

We repeat the learning procedure in the setting where training data is partially observed at every two out of three coordinates. The results are shown in \cref{fig:l96_learn_partial}. Those for AD-PF-T are not shown since training cannot be completed due to filter divergence. We find that AD-EnKF-T is still able to recover $f^*$ consistently as well as the unknown states for all coordinates, including the ones that are not observed, and has a filter RMSE close to the one computed with knowledge of $f^*$. However, the performance metrics of the EM algorithm in the model correction experiment deteriorate as training proceeds, indicating that it may overfit the training data. In addition, we find that the EM algorithm does not converge to the same point in repeated trials, particularly so in the setting of fully unknown dynamics. All of these results indicate that AD-EnKF is advantageous when learning from partial observations in high dimensions.

The ability to recover the underlying dynamics and states even with incomplete observations and fully unknown dynamics is most likely due to the convolutional-type architecture of the NN $f_\alpha^\NN$, which implicitly assumes that each coordinate only interacts with its neighbors, and that this interaction is spatially invariant.

\begin{figure}[htbp]
	\centering
	\includegraphics[width=\textwidth]{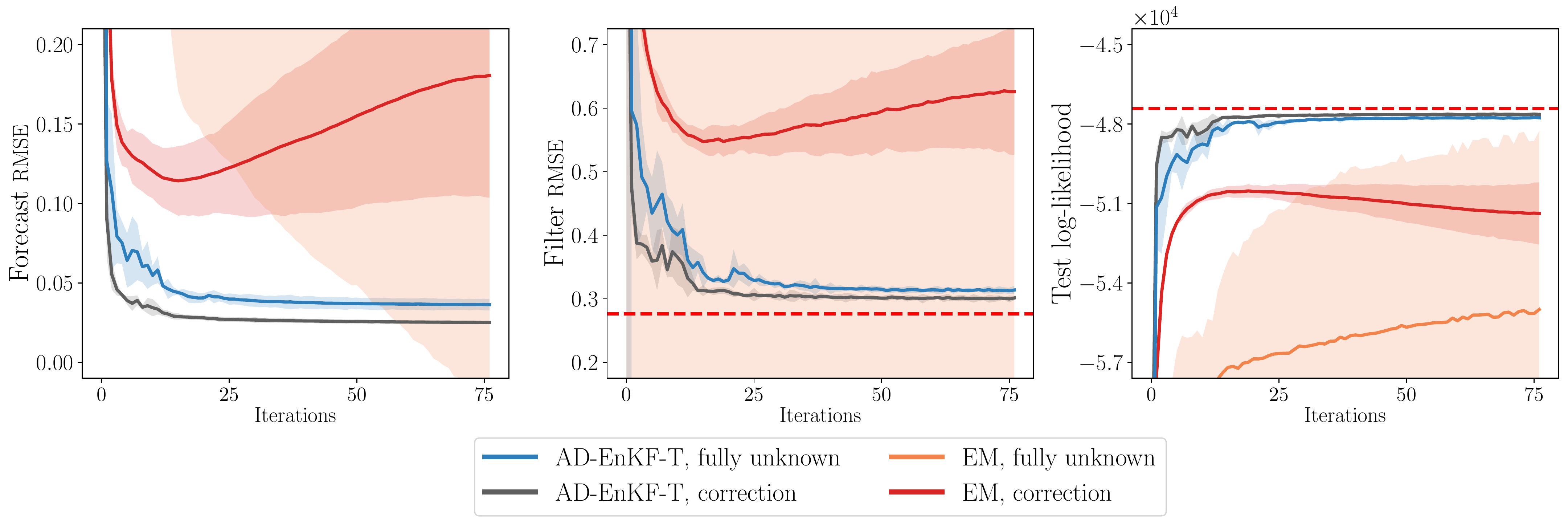}
    \vspace*{-7mm}
    \caption{Learning Lorenz-96 from fully unknown dynamics (\cref{ssec:l96-unknown}) v.s.\ model correction (\cref{ssec:l96-correction}) with partial observations ($H=[e_1, e_2, e_4, e_5, e_7, \cdots]^{\top}$). All performance metrics are evaluated after each training iteration. Red dashed lines correspond to metric values obtained with the reference model $f^*$ and $Q^*$. The absence of lines for EM in the fully unknown setting is due to its low and unstable performance.
    }
	 \label{fig:l96_learn_partial}
\end{figure}


\section{Conclusions and Future Directions}\label{sec:conclusions}
This paper introduced AD-EnKFs for the principled learning of states and dynamics in DA. We have shown that AD-EnKFs can be successfully integrated with DA localization techniques for recovery of high-dimensional states, and with TBPTT techniques to handle large observation data and high-dimensional surrogate models. Numerical results on the Lorenz-96 model show that AD-EnKFs outperform existing EM and PF methods to merge DA and ML. 

Several research directions stem from this work. First, gradient and Hessian information of $\LhEnKF$ obtained by autodiff can be utilized to design optimization schemes beyond the first-order approach we consider. Second, the convergence analysis of EnKF estimation of the log-likelihood and its gradient may be generalized to nonlinear settings. Third, the idea of AD-EnKF could be applied to auto-differentiate through other filtering algorithms, e.g. unscented Kalman filters, and 
in Bayesian inverse problems using iterative ensemble Kalman methods. Finally, the encouraging numerical results obtained on the Lorenz-96 model motivate the deployment and further investigation of AD-EnKFs in scientific and engineering applications where latent states need to be estimated with incomplete knowledge of their dynamics. 

\section*{Acknowledgments}
YC was partially supported by DMS-2027056 and NSF OAC-1934637.
DSA is grateful for the support of DMS-2027056. 
RW is grateful for the support of DOD FA9550-18-1-0166, DOE DE-AC02-06CH11357, NSF OAC-1934637, DMS-1930049, and DMS-2023109. 

\bibliographystyle{plain}
\bibliography{references}

\appendix

\paragraph{Notation} We denote by $c$ a constant that does not depend on $N$ and may change from line to line. We denote by $\|U\|_p$ the $L^p$ norm of a random vector/matrix $U$: $\|U\|_p \triangleq \big(\Expect|U|^p\big)^{1/p}$, where $|\cdot|$ is the underlying vector/matrix norm. (Here we use 2-norm for vectors and Frobenius norm for matrices.) For a sequence of random vectors/matrices $U_N$, we write
\[
U_N\conv U
\]
if, for any $p \ge 1$, there exists a constant $c$ such that
\begin{equation*}
\|U_N - U\|_p \le cN^{-1/2}, \quad\quad \forall N \ge 1.
\end{equation*}

For a scalar valued function $f(U)$ that takes a vector/matrix $U$ as input, we denote by $\partial_U f$ the derivative of $f$ w.r.t.\ $U$, which collects the derivative of $f$ w.r.t.\ each entry of the vector/matrix $U$. When $U$ is a vector, the notations $\partial_U f$ and $\nabla_U f$ are equivalent.

For a vector/matrix valued function $U(a)$ that takes a scalar $a$ as input, we denote by $\partial_a U$ the derivative of $U$ w.r.t.\ $a$, which collects the derivative of each entry of the vector/matrix $U$ w.r.t.\ $a$.


\section{Proof of Theorem \ref{prop:likelihood}}\label{sec:likeproof-appen}

We first recall the propagation of chaos statement. Notice that in the EnKF algorithm \cref{alg:EnKF}, we compute $\ton{x}$ sequentially, based on the forecast ensemble $\widehat x_t^{1:N}$ and its empirical mean and covariance $\tf{m}, \tf{C}$. We build ``substitute particles'' $\ton{\mf x}$ in a similar fashion, except that at each step the population mean and covariance $\tf{\mf m}, \tf{\mf C}$ are used instead of their empirical versions. Starting from $\mf x_0^{1:N} = x_0^{1:N}$, the update rules of substitute particles are listed below, with a side-by-side comparison to the EnKF update rules:
\begin{equation}\label{eq:sub_def}
\setstretch{1.4}
\begin{array}{rl|rl}
\multicolumn{2}{c|}{\text{EnKF particles}} & \multicolumn{2}{c}{\text{Substitute particles}}\\
\hline
\tnf{x} &= F_\alpha (\tmn{ x}) + \tn{\xi} &\tnf{\mf x} &= F_\alpha (\tmn{\mf x}) + \tn{\xi}\\
 \tf{m} &= \frac{1}{N} \sum_{n=1}^N \tnf{x} &\tf{\mf m} &= \Expect \bigl[\tnf{\mf x}\bigr]\\
\tf{C} &=  \frac{1}{N-1} \sum_{n=1}^N (\tnf{x} - \tf{m}) (\tnf{x} - \tf{m})^{\top}  & \tf{\mf C} &= \Expect \bigl[(\tnf{\mf x} - \tf{\mf m})(\tnf{\mf x} - \tf{\mf m})^{\top}\bigr]\\
\widehat K_t &= \tf{C} H^{\top} (H  \tf{C} H^{\top} + R)^{-1} & \widehat {\mf K}_t &= \tf{\mf C} H^{\top} (H \tf{\mf C} H^{\top} + R)^{-1}\\
\tn{ x} &= \tnf{ x} +  \widehat K_t (y_t + \tn{\gamma} - H \tnf{ x}) & \tn{\mf x} &= \tnf{\mf x} + \widehat {\mf K}_t (y_t + \tn{\gamma} - H \tnf{\mf x})\\
\end{array}   
\end{equation}
Notice that the substitute particles use the \emph{same} realization of random variables as the EnKF particles, including initialization of particles $x_0^{1:N}$, forecast simulation error $\tn{\xi},$ and noise perturbation $\tn{\gamma}$. As $N\rightarrow\infty$, one can show that the  EnKF particles $\ton{x}$ (resp. $\ton{\widehat x}$) are close to the substitute particles $\ton{\mf x}$ (resp. $\ton{\widehat {\mf x}}$), and hence the law of large numbers guarantees that $\tf{m}, \tf{C}$ are close to $\tf{\mf m}, \tf{\mf C}$. We summarize the main results from \cite{le2009large} (see also \cite{kwiatkowski2015convergence}): 
\begin{lemma}\label{lemma:main}
Under the same assumption of \cref{prop:likelihood}:\\
(1) For each $t\ge 1$, the substitute particles $\ton{\mf x}$ are i.i.d., and each of them has the same law as the true filtering distribution $p(x_t|y_{1:t})$. Similarly, $\tonf{\mf x}$ are i.i.d., and each of them has the same law as the true forecast distribution $p(x_t|y_{1:t-1})$. In particular,
\begin{equation}\label{eq:sub_lemma_1}
    p(x_t|y_{1:t-1}) = \Nc(x_t; \tf{\mf m}, \tf{\mf C}).
\end{equation}
(2) For each $t, n, p \ge 1,$ the EnKF particle $\tn{x}$ converges to the substitute particle $\tn{\mf x}$ in $L^p$ with convergence rate $N^{-1/2}$, and the substitute particle $\tn{\mf x}$ has finite moments of any order. The same holds for forecast particles $\tnf{x}$:
\begin{equation}\label{eq:sub_lemma_2}
    \tn{x} \conv \tn{\mf x}, \quad\quad \tnf{x} \conv \tnf{\mf x}, \quad\quad \|\tn{\mf x}\|_p \le c, \quad\quad \|\tnf{\mf x}\|_p \le c.
\end{equation}
In particular, $\widehat{m}_t$, $\widehat{C}_t$ converge to $\tf{\mf m}$, $\tf{\mf C}$ in $L^p$ with convergence rate $N^{-1/2}$:
\begin{equation}\label{eq:sub_lemma_3}
    \tf{m} \conv \tf{\mf m}, \quad\quad\quad\quad \tf{C} \conv \tf{\mf C}.
\end{equation}
\end{lemma}
\begin{proof}
\cref{eq:sub_lemma_1} corresponds to Lemma 2.1 of \cite{le2009large}. \cref{eq:sub_lemma_2} corresponds to Proposition 4.4 of \cite{le2009large}. \cref{eq:sub_lemma_3} is a direct corollary of Theorem 5.2 of \cite{le2009large}.
\end{proof}

\begin{proof}[Proof of {\hypersetup{hidelinks}\cref{prop:likelihood}}]
By \cref{eq:sub_lemma_1}, using the Gaussian observation assumption \cref{eq:main_obs}:
\begin{equation}\label{eq:KF_like}
   \L(\theta) =\sum_{t=1}^T \log p(y_t|y_{1:t-1}) = \sum_{t=1}^T \log \Nc(y_t; H \tf{\mf m}, H \tf{\mf C} H^{\top} + R).
\end{equation}
By \cref{eq:EnKF_loglike},
\begin{equation}
    \LhEnKF(\theta) = \sum_{t=1}^T \log \Nc (y_t; H \tf{m}, H \tf{C} H^{\top} + R).
\end{equation}
Define
\begin{equation}\label{eq:ht_proof}
\begin{split}
    h_t(m, C) &\triangleq \log \Nc(y_t; Hm, HCH^{\top}+R)\\
        &= -\frac{1}{2} \log \det (HCH^{\top}+R) - \frac{1}{2} (y_t - H m)^{\top} (HCH^{\top}+R)^{-1} (y_t - H m) + \const.
\end{split}
\end{equation}
It suffices to show that, for each $t\ge 1$,
\begin{equation}\label{eq:ht-conv-proof}
    h_t(\tf{m}, \tf{C}) \conv h_t(\tf{\mf m}, \tf{\mf C}) .
\end{equation}
We denote by $\Ss_+^{d_x} \subset \R^{d_x \times d_x}$ the space of all positive semi-definite matrices equipped with Frobenius norm. Notice that $h_t$ is a continuous function on $\R^{d_x} \times \Ss_+^{d_x}$, since $HCH^\top + R \succeq R \succ 0$. To show convergence in $L^p$, intuitively one would expect a Lipschitz-type continuity to hold for $h_t$, in a suitable sense. We inspect the derivatives of $h_t$ w.r.t.\ $m$ and $C$, which will also be useful for later developments:
\begin{equation}\label{eq:grad-h}
\begin{split}
    \partial_m h_t (m, C) &= - H^{\top} (HCH^{\top} + R)^{-1} (y_t - Hm), \\
    \partial_C h_t (m, C) &= -\frac{1}{2} H^{\top} (HCH^{\top}+R)^{-1} H \\
    &\quad\quad\quad + \frac{1}{2} H^{\top} (HCH^{\top}+R)^{-1} (y_t-Hm)(y_t-Hm)^{\top} (HCH^{\top}+R)^{-1} H.
\end{split}
\end{equation}
Since $R^{d_x}\times \Ss_+^{d_x}$ is convex, by the mean value theorem, triangle inequality and Cauchy-Schwarz, and define $m(\chi)\triangleq  \chi\tf{\mf m} + (1-\chi) \tf{m}$, $C(\chi)\triangleq  \chi\tf{\mf C} + (1-\chi) \tf{C}$,
\begin{equation}
\begin{split}
    \big| h_t(\tf{m}, \tf{C}) - h_t(\tf{\mf m}, \tf{\mf C}) \big| &\le  \sup_{\chi \in [0,1]} \big| \partial_m h_t \big( m(\chi), C(\chi) \big) \big| |\tf{m} - \tf{\mf m}| \\
    & \quad\quad\quad + \sup_{\chi \in [0,1]} \big| \partial_C h_t \big( m(\chi), C(\chi) \big) \big| |\tf{C} - \tf{\mf C}|.
\end{split}
\end{equation}
Taking $L_p$ norm on both sides,
\begin{equation}\label{eq:prop_h_p}
\begin{split}
    \big\| h_t(\tf{m}, \tf{C}) - h_t(\tf{\mf m}, \tf{\mf C}) \big\|_p &\le  \sup_{\chi \in [0,1]} \big\| \partial_m h_t \big( m(\chi), C(\chi) \big) \big\|_{2p} \|\tf{m} - \tf{\mf m}\|_{2p} \\
    & \quad\quad\quad + \sup_{\chi \in [0,1]} \big\| \partial_C h_t \big( m(\chi), C(\chi) \big) \big\|_{2p} \|\tf{C} - \tf{\mf C}\|_{2p},
\end{split}
\end{equation}
where we have used the triangle inequality and the $L^p$ Cauchy-Schwarz inequality $\| |U||V| \|_p \le \|U\|_{2p} \|V\|_{2p}$, see e.g., Lemma 2.1 of \cite{kwiatkowski2015convergence}. Also, by plugging in $\cref{eq:grad-h}$, for each $\chi \in [0,1]$,
\begin{equation}\label{eq:prop_hm_p}
\begin{split}
    \big\| \partial_m h_t \big( m(\chi), C(\chi) \big) \big\|_{2p} &\le \big\| |H| |(H C(\chi) H^{\top}+R)^{-1}|(|y_t-\chi H \tf{\mf m} - (1-\chi)H \tf{m}|) \big\|_{2p}\\
    &\le |H||R^{-1}|\big(|y_t| + |H| |\tf{\mf m}| + |H|\|\tf{m}\|_{2p} \big) \le c,
\end{split}
\end{equation}
where we have used that $|(HC(\chi)H^{\top}+R)^{-1}| \le |R^{-1}|$, that $\tf{\mf m}$ is deterministic, and that all moments of $\tf{m}$ are finite, by \cref{eq:sub_lemma_3}. Similarly,
\begin{equation}\label{eq:prop_hc_p}
\begin{split}
    \big\| \partial_C h_t \big( m(\chi), C(\chi) \big) \big\|_{2p} & \le \frac{1}{2} |H|^2|R^{-1}| + \frac{1}{2} |H|^2|R^{-1}|^2 \|y_t - \chi H \tf{\mf m} - (1-\chi)H \tf{m} \|_{4p}^2\\
    &\le \frac{1}{2} |H|^2|R^{-1}| + \frac{1}{2} |H|^2|R^{-1}|^2(|y_t|+|H||\tf{\mf m}| + |H|\|\tf{m}\|_{4p})^2 \le c,
\end{split}
\end{equation}
where we have  used that $|vv^{\top}| = |v|^2$ for vector $v$. Thus, combining \cref{eq:prop_h_p,eq:prop_hm_p,eq:prop_hc_p,eq:sub_lemma_3} gives
\begin{equation}
\| h_t(\tf{m}, \tf{C}) - h_t(\tf{\mf m}, \tf{\mf C}) \|_p \le c N^{-1/2},
\end{equation}
which concludes the proof.
\end{proof}

\section{Proof of Theorem \ref{prop:gradient}}\label{sec:gradproof-appen}
Without loss of generality, we assume that $\theta\in \R$ is a scalar parameter, since in general the gradient w.r.t.\ $\theta$ is a collection of derivatives w.r.t.\ each element of $\theta$. We will use the following lemma repeatedly:
\begin{lemma}\label{lemma:conv}
For sequences of random vectors/matrices $U_N$, $V_N$:\\
(1) If $U_N - V_N \conv 0$ and $V_N \conv V,$ then 
\begin{equation}
U_N \conv V.
\end{equation}
(2) If $U_N \conv U$, $V_N \conv V$ and $U, V$ have finite moments of any order, then
\begin{equation}
    U_N V_N \conv U V.
\end{equation}
More generally, the result holds for multiplication of more than two variables.
\end{lemma}
\begin{proof}
(1) Using $U_N-V=(U_N-V_N)+(V_N-V)$, the proof follows from the triangle inequality.\\
(2) Applying triangle inequality and $L^p$ Cauchy-Schwarz inequality,
\begin{equation}
\begin{split}
\|U_N V_N - UV\|_p &\le \|(U_N-U)V_N\|_p+ \|U(V_N-V)\|_p\\
&\le \|U_N-U\|_{2p} \|V_N\|_{2p} + \|U\|_{2p}\|V_N-V\|_{2p}\\
&\le c N^{-1/2} \Bigl(\|V\|_{2p}+c N^{-1/2} \Bigr) + \|U\|_{2p}c N^{-1/2}\\ 
&\le c N^{-1/2}.
\end{split}
\end{equation}
\end{proof}

The following result, which we will use repeatedly, is an immediate corollary of \cref{lemma:main}:
\begin{lemma}\label{lemma:HCHT}
Under the same assumption of \cref{prop:likelihood},
\begin{equation}\label{eq:HCHT-proof}
    (H\tf{C} H^\top+R)^{-1} \conv (H \tf{\mf C} H^\top+R)^{-1}.
\end{equation}
\end{lemma}
\begin{proof} Using the identity $A^{-1}-B^{-1} = A^{-1}(B-A)B^{-1}$ for invertible matrices $A$, $B$:
\begin{equation}
\begin{split}
    &\phantom{{}={}}\|(H \tf{C} H^\top +R)^{-1} - (H \tf{\mf C} H^\top +R)^{-1}\|_p \\
    &= \|(H \tf{C} H^\top +R)^{-1} H (\tf{\mf C} - \tf{C})H^\top (H \tf{\mf C} H^\top +R)^{-1}\|_p \\
    &\le  |R^{-1}|^2|H|^2  \| \tf{C} - \tf{\mf C} \|_{p}\\
    &\le c N^{-1/2},
\end{split}
\end{equation}
where we have used the $L^p$ convergence of $\tf{C}$ to $\tf{\mf C}$ \cref{eq:sub_lemma_3}, and the fact that $| (H C H^T+R)^{-1}| \le |R^{-1}|$ for $C \succeq 0$.
\end{proof}

\begin{lemma}\label{lemma:conv-grad}
Under the same assumption of \cref{prop:gradient}, for each $t\ge 1$, both $\pt \tnf{x}$ and $\pt{\tnf{\mf x}}$ exist, and $\pt \tnf{x}$ converges to $\pt \tnf{\mf x}$ in $L^p$ for any $p\ge 1$ with convergence rate $N^{-1/2}$. Moreover, $\pt \tnf{\mf x}$ has finite moments of any order:
\begin{gather}
    \pt \tnf{x} \conv \pt \tnf{\mf x},  \quad\quad
    \| \pt \tnf{\mf x} \|_{p} \le c, \quad\quad \forall n. \label{eq:grad_lemma_1}
\end{gather}
In addition, all derivatives $\pt \tf{m}$, $\pt \tf{\mf m}$, $\pt \tf{C}$, $\pt \tf{\mf C}$, $\pt \tf{K}$ and $\pt \tf{\mf K}$ exist, and 
\begin{gather}
     \pt \tf{m} \conv \pt \tf{\mf m},\quad\quad
     \pt \tf{C} \conv \pt \tf{\mf C},\quad\quad
     \pt \widehat K_t \conv \pt \widehat {\mf K}_t. \label{eq:grad_lemma_2}
\end{gather}
\end{lemma}
\begin{proof}
We will prove this by induction. For $t=1$, since $\onf{x} = A x_0^n+S \xi_0^n = \onf{\mf x}$,
\begin{equation}
    \pt \onf{x} = (\pt A) x_0^n + (\pt S) \xi_0^n = \pt \onf{\mf x},
\end{equation}
and both derivatives $\pt \onf{x}$ and $\pt \onf{\mf x}$ exist.  Also,
\begin{equation}
    \| \pt \widehat{\mf x}_1^n \|_{p} \le |\pt A| \|x_0^n\|_p + |\pt S| \|\xi_0^n\|_p \le c,
\end{equation}
since $x_0^n$ and $\xi_0^n$ are drawn from Gaussian distributions, which have finite moments of any order. So \cref{eq:grad_lemma_1} holds for $t=1$.

Assume \cref{eq:grad_lemma_1} holds for step $t$. Then, using the definition for $\tf{m}$:
\begin{equation}
    \pt \tf{m} = \frac{1}{N}   \sum_{n=1}^N \pt\tnf{x} \convincircle{1} \frac{1}{N}   \sum_{n=1}^N \pt\tnf{\mf x} \convincircle{2} \Expect[\pt \tnf{\mf x}] \underset{\incircle{3}}{=} \pt \Expect[\tnf{x}] = \pt \tf{\mf m}.
\end{equation}
Convergence \incircle{1} follows from induction assumption \cref{eq:grad_lemma_1}. Convergence \incircle{2} follows from law of large numbers in $L^p$, since $\pt \tnf{\mf x}$ are i.i.d. and the moments of $\pt \tnf{\mf x}$ are finite by induction assumption \cref{eq:grad_lemma_1}. The swap of differentiation and expectation in \incircle{3} is valid since the expectation is taken over a distribution that is independent of $\theta$. Both derivatives $\pt \tf{m}$ and $\pt \tf{\mf m}$ exist. Similarly,
\begingroup
\allowdisplaybreaks
\begin{equation}
\begin{split}
    \pt \tf{C} &\Hquad=\Hquad \frac{1}{N-1} \sum_{n=1}^N \pt(\tnf{x} (\tnf{x})^\top) - \frac{N}{N-1}\pt(\tf{m}\tf{m}^\top)\\
    &\Hquad=\Hquad \frac{1}{N-1} \sum_{n=1}^N \big( (\pt \tnf{ x}) (\tnf{ x})^{\top} + \tnf{ x}(\pt \tnf{ x})^{\top} \big) - \frac{N}{N-1} \big((\pt \tf{m}) \tf{m}^{\top} + \tf{m} (\pt \tf{m})^{\top} \big) \\
    &\convincircle{1} \frac{1}{N-1} \sum_{n=1}^N \big( (\pt \tnf{\mf x}) (\tnf{\mf x})^{\top} + \tnf{  \mf x}(\pt \tnf{\mf x})^{\top} \big) - \frac{N}{N-1} \big((\pt \tf{\mf m}) \tf{\mf m}^{\top} + \tf{\mf m} (\pt \tf{\mf m})^{\top} \big) \\
    &\convincircle{2} \E[\pt \big(\tnf{\mf x}(\tnf{\mf x} \big)^{\top})] - \pt (\tf{\mf m} \tf{\mf m}^{\top})\\
    &\Hquad\underset{\incircle{3}}{=}\Hquad \pt (\Expect[\tnf{\mf x} (\tnf{\mf x})^{\top}] - \tf{\mf m} \tf{\mf m}^{\top}) \\
    &\Hquad=\Hquad \pt \tf{\mf C}.
\end{split}
\end{equation}
\endgroup
For \incircle{1} we have used the $L^p$ convergence of $\tnf{x}$ to $\tnf{\mf x}$, $\pt \tnf{x}$ to $\pt \tnf{\mf x}$, $\tf{m}$ to $\tf{\mf m}$ and $\pt \tf{m}$ to $\pt \tf{\mf m}$ with rate $N^{-1/2}$, and the fact that $\tnf{\mf x}$ and $\pt \tnf{\mf x}$ have finite moments of any order, followed by \cref{lemma:conv}. Convergence \incircle{2} follows from law of large numbers in $L^p$ since $(\pt \tnf{\mf x})(\tnf{\mf x})^\top$ are i.i.d. with finite moments, by the Cauchy-Schwarz inequality. \incircle{3} is valid since the expectation is taken over a distribution that is independent of $\theta$. Both derivatives $\pt \tf{C}$ and $\pt \tf{\mf C}$ exist. Similarly,
\begin{equation}
\begin{split}
    \pt \widehat K_t &\Hquad=\Hquad \pt \big( \tf{C} H (H \tf{C} H^{\top} + R)^{-1} \big) \\
    &\Hquad\underset{\incircle{1}}{=}\Hquad  (\pt \tf{C}) H (H \tf{C} H^{\top}+R)^{-1} \\
    &\quad\quad\quad - \tf{C} H (H \tf{C} H^{\top} + R)^{-1} (H (\pt \tf{C}) H^{\top} + R) (H \tf{C} H^{\top} + R)^{-1} \\
    &\convincircle{2} (\pt \tf{\mf C}) H (H \tf{\mf C} H^{\top}+R)^{-1} \\
    &\quad\quad\quad - \tf{\mf C} H (H \tf{\mf C} H^{\top} + R)^{-1} (H (\pt \tf{\mf C}) H^{\top} + R) (H \tf{\mf C} H^{\top} + R)^{-1} \\
    &\Hquad\underset{\incircle{1}}{=}\Hquad \pt \big( \tf{\mf C} H (H \tf{\mf C} H^{\top} + R)^{-1} \big) \\
    &\Hquad=\Hquad \pt \widehat {\mf K}_t.
\end{split}
\end{equation}
Here equalities \incircle{1} and \incircle{3} follow from chain rule. For \incircle{2} we have used the $L^p$ convergence of $\tf{C}$ to $\tf{\mf C}$, $\pt \tf{C}$ to $\pt \tf{\mf C}$ and $(H\tf{C} H^\top+R)^{-1}$ to $(H\tf{\mf C} H^\top+R)^{-1}$ with rate $N^{-1/2}$ (by \cref{eq:HCHT-proof}), followed by \cref{lemma:conv}. Both derivatives $\pt \widehat K_t$ and $\pt \widehat{\mf K}_t$ exist since $R\succ 0$.

To show \cref{eq:grad_lemma_1} holds for step $t+1$, we need to investigate the derivatives of the analysis ensemble $\pt \tn{x}$, by plugging in the EnKF update formula:

\begin{equation}
\begin{split}
    \pt \tn{x} &\Hquad=\Hquad \pt (\tnf{x} + \widehat K_t (y_t + \tn{\gamma} - H \tnf{x}) \\
    &\Hquad\underset{\incircle{1}}{=}\Hquad (I - \widehat K_t H) \pt \tnf{x} + (\pt \widehat K_t) (y_t + \tn{\gamma} - H \tnf{x}) \\
    &\Hquad=\Hquad \big(I - \tf{C} H^\top (H \tf{C} H^\top + R)^{-1} H\big) \pt \tnf{x} + (\pt \widehat K_t) (y_t + \tn{\gamma} - H \tnf{x}) \\
    &\convincircle{2} (I- \tf{\mf C} H^\top (H \tf{\mf C} H^\top + R)^{-1} H) \pt \tnf{\mf x} + (\pt \widehat {\mf K}_t) (y_t + \tn{\gamma} - H \tnf{\mf x})\\
    &\Hquad=\Hquad (I - \widehat{\mf K}_t H) \pt \tnf{\mf x} + (\pt \widehat{\mf K}_t) (y_t + \tn{\gamma} - H \tnf{\mf x}) \\
    &\Hquad\underset{\incircle{3}}{=}\Hquad \pt (\tnf{\mf x} + \widehat {\mf K}_t (y_t + \tn{\gamma} - H \tnf{\mf x}) \\
    &\Hquad=\Hquad \pt \tn{\mf x}.
\end{split}
\end{equation}
Equalities \incircle{1} and \incircle{3} follow from chain rule. For \incircle{2} we have used the $L^p$ convergence of $\tnf{x}$ to $\tnf{\mf x}$, $\pt \tnf{x}$ to $\pt \tnf{\mf x}$, $\tf{C}$ to $\tf{\mf C}$, $\pt \widehat K_t$ to $\pt \widehat{\mf K}_t$, and $(H\tf{C} H^\top+R)^{-1}$ to $(H\tf{\mf C} H^\top+R)^{-1}$, with convergence rate $N^{-1/2}$, and the fact that $\tnf{\mf x}$, $\pt \tnf{\mf x}$ and the Gaussian random variable $\tn{\gamma}$ have finite moments of any order, followed by \cref{lemma:conv}. Both derivatives $\pt \tn{x}$ and $\pt \tn{\mf x}$ exist since $R \succ 0$. We also have the moment bound on $\pt \tn{\mf x}$:
\begin{equation}
    \| \pt \tn{\mf x} \| \le |I - \widehat {\mf K}_t H| \| \pt \tnf{\mf x} \|_p + |\pt \widehat {\mf K}_t| (|y_t| + \|\tn{\gamma}\|_p + |H| \|\tnf{\mf x}\|_p) \le c,
\end{equation}
since $\tnf{\mf x}$, $\pt \tnf{\mf x}$ and the Gaussian random variable $\tn{\gamma}$ have finite moments of any order. Then,
\begin{equation}
\begin{split}
    \pt \tpnf{x} &\Hquad=\Hquad \pt (A \tn{x} + S \tn{\xi})\\
    &\Hquad\underset{\incircle{1}}{=}\Hquad (\pt A) \tn{x} + A (\pt \tn{x}) + (\pt{S}) \tn{\xi}\\
    &\convincircle{2} (\pt A) \tn{\mf x} + A (\pt \tn{\mf x}) + (\pt S) \tn{\xi}\\
    &\Hquad\underset{\incircle{3}}{=}\Hquad \pt (A \tn{\mf x} + S \tn{\xi})\\
    &\Hquad=\Hquad \pt \tpnf{\mf x}.
\end{split}
\end{equation}
Here equalities \incircle{1} and \incircle{3} follow from chain rule. For \incircle{2} we have used the $L^p$ convergence of $\tn{x}$ to $\tn{\mf x}$ and $\pt \tn{x}$ to $\pt \tn{\mf x}$. Both derivatives $\pt \tpnf{x}$ and $\pt \tpnf{\mf x}$ exist since both $\pt \tn{x}$ and $\pt \tn{\mf x}$ exist. We also have the moment bound:
\begin{equation}
\begin{split}
    \| \pt \tpnf{\mf x} \|_p & \le |\pt A| \| \tn{\mf x} \|_p + |A| \| \pt \tn{\mf x} \|_p + |\pt S| \| \tn{\xi} \|_p \le c,
\end{split}
\end{equation}
since $\tn{\mf x}$, $\pt \tn{\mf x}$ and Gaussian random variable $\tn{\xi}$ have finite moments of any order.
Thus \cref{eq:grad_lemma_1} is proved for step $t+1$ and the induction step is finished, which concludes the proof of the lemma.
\end{proof}
\begin{proof}[Proof of {\hypersetup{hidelinks}\cref{prop:gradient}}]
Recall the definition of $h_t$ \cref{eq:ht_proof}. It suffices to show that, for each $t\ge 1$,
\begin{equation}
    \pt \Big(h_t(\tf{m}, \tf{C}) \Big) \conv \pt \Big(h_t (\tf{\mf m}, \tf{\mf C}) \Big).
\end{equation}
We first investigate the convergence of derivatives of $h_t$ w.r.t $\tf{m}$ and $\tf{C}$. The derivatives are computed in \cref{eq:grad-h}:
\begin{equation}\label{eq:partialm-conv}
\begin{split}
    \partial_m h_t(\tf{m}, \tf{C}) &\Hquad=\Hquad - H^{\top} (H \tf{C} H^{\top} + R)^{-1} (y_t - H \tf{m})\\
    &\conv - H^{\top} (H \tf{\mf C} H^{\top} + R)^{-1} (y_t - H \tf{\mf m})\\
    &\Hquad=\Hquad \partial_m h_t(\tf{\mf m}, \tf{\mf C}),
\end{split}
\end{equation}
and
\begin{equation}\label{eq:partialc-conv}
\begin{split}
    \partial_C h_t(\tf{m}, \tf{C}) &\Hquad=\Hquad -\frac{1}{2} H^{\top} (H\tf{C} H^{\top}+R)^{-1} H\\
    &\quad\quad\quad\quad + \frac{1}{2} H^{\top} (H\tf{C} H^{\top}+R)^{-1} (y_t-H \tf{m})(y_t-H\tf{m})^{\top} (H\tf{C} H^{\top}+R)^{-1} H\\
    &\conv -\frac{1}{2} H^{\top} (H\tf{\mf C} H^{\top}+R)^{-1} H \\
    &  \quad\quad\quad\quad + \frac{1}{2} H^{\top} (H\tf{\mf C} H^{\top}+R)^{-1} (y_t-H\tf{\mf m})(y_t-H \tf{\mf m})^{\top} (H\tf{\mf C} H^{\top}+R)^{-1} H,
\end{split}
\end{equation}
where we have used the $L^p$ convergence of $\tf{m}$ to $\tf{\mf m}$ and $(H\tf{C} H^\top+R)^{-1}$ to $(H \tf{\mf C} H^\top+R)^{-1}$ by \cref{eq:HCHT-proof}, followed by \cref{lemma:conv}. Then, by chain rule,
\begin{equation}
\begin{split}
    \pt \Big( h_t(\tf{m}, \tf{C}) \Big) &\Hquad=\Hquad \Big( \partial_m h_t(\tf{m}, \tf{C}) \Big)^\top \pt \tf{m} + \text{Tr}\bigg( \Big( \partial_C h_t(\tf{m}, \tf{C})\Big)^\top \pt \tf{C} \bigg)\\
   &\conv \Big( \partial_m h_t(\tf{\mf m}, \tf{\mf C}) \Big)^\top \pt \tf{\mf m} + \text{Tr}\bigg( \Big( \partial_C h_t(\tf{\mf m}, \tf{\mf C})\Big)^\top \pt \tf{\mf C} \bigg)\\
    &\Hquad=\Hquad \pt \Big( h_t(\tf{\mf m}, \tf{\mf C}) \Big).
\end{split}
\end{equation}
Both derivatives exist since $\pt \tf{m}$, $\pt \tf{\mf m}$, $\pt \tf{C}$ and $\pt \tf{\mf C}$ exist, by \cref{lemma:conv-grad}. We have used \cref{eq:partialm-conv} and \cref{eq:partialc-conv} above, the $L^p$ convergence of $\pt \tf{m}$ to $\pt \tf{\mf m}$ and $\pt \tf{C}$ to $\pt \tf{\mf C}$ with rate $N^{-1/2}$ by \cref{lemma:conv-grad}, followed by \cref{lemma:conv}.
\end{proof}
\begin{remark}
We again emphasize that all the derivatives and chain rule formulas do \emph{not} need to be computed by hand in applications, but rather through the modern autodiff libraries. We list them out only for the purpose of proving convergence results. 
\end{remark}

\section{Additional Figures}\label{sec:addfig-appen}
See \cref{fig:linear_converge_off,fig:linear_converge_loc_off} for additional figures to the linear Gaussian experiment \cref{sec:linear}:
\begin{figure}[htbp]
	\centering
	\includegraphics[width=\textwidth]{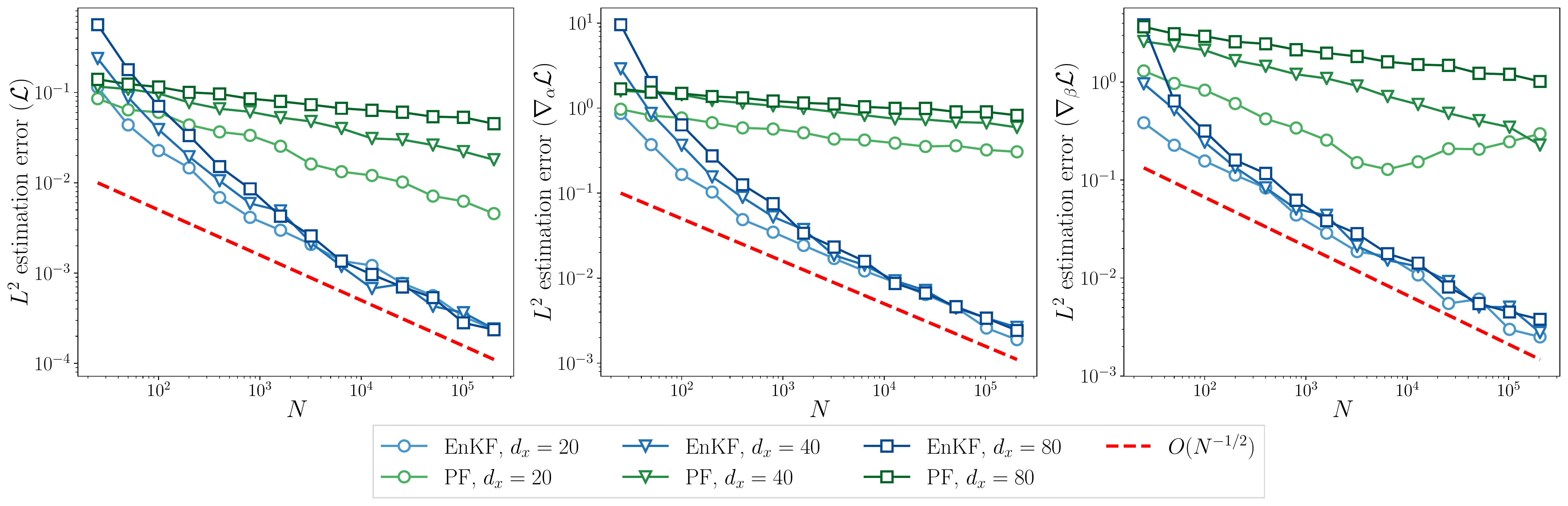}
	\vspace*{-7mm}
	\caption{Relative $L^2$ estimation errors of the log-likelihood (left) and its gradient w.r.t.\ $\alpha$ (middle) and $\beta$ (right), computed using EnKF and PF, as a function of $N$, for the linear-Gaussian model  \eqref{eq:linGauss}. State dimension $d_x\in\{20,40,80\}$. $\theta$ is evaluated at $\alpha=(0.5, 0.5, 0.5), \beta=(1,0.1)$. (\cref{sec:linear_converge})}
	 \label{fig:linear_converge_off}
\end{figure}

\begin{figure}[htbp]
	\centering
	\includegraphics[width=\textwidth]{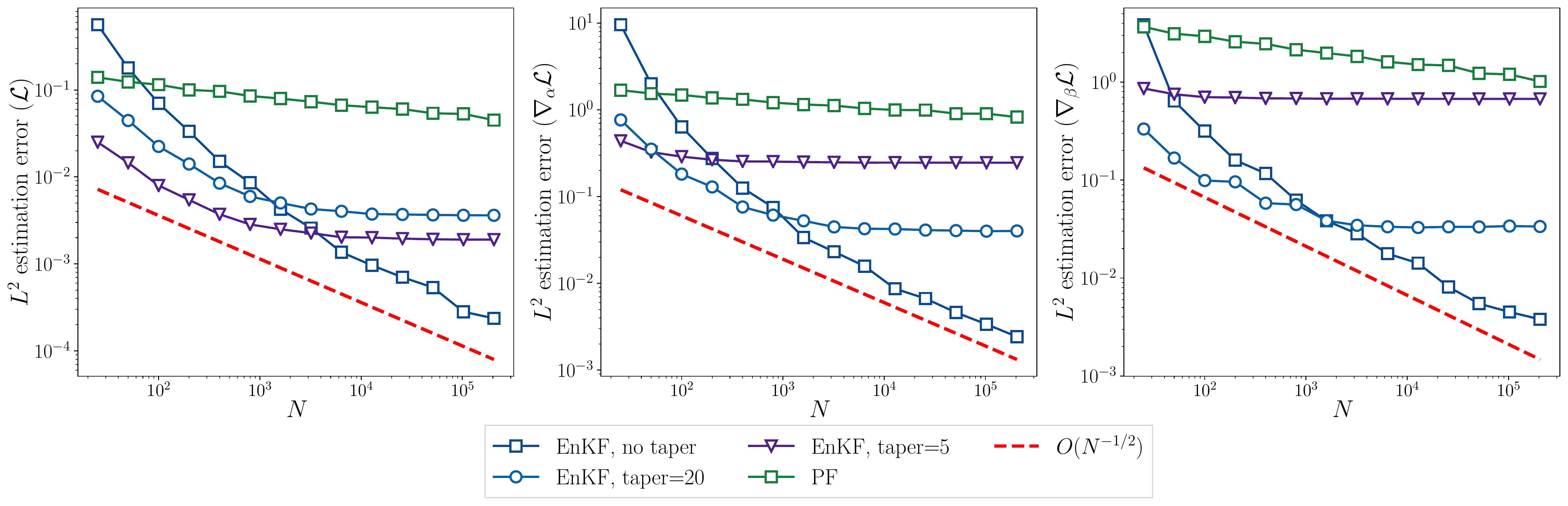}
	\vspace*{-7mm}
	\caption{Relative $L^2$ estimation errors of log-likelihood (left) and its gradient w.r.t.\ $\alpha$ (middle) and $\beta$ (right), computed using EnKF and PF, with different covariance tapering radius applied to EnKF. State dimension $d_x=80$. $\theta$ is evaluated at $\alpha=(0.5, 0.5, 0.5), \beta=(1,0.1)$.}
	 \label{fig:linear_converge_loc_off}
\end{figure}


\section{Auto-Differentiable Particle Filters}\label{sec:pf-appen}
Here we review the auto-differentiable particle filter methods introduced in, e.g.  \cite{naesseth2018variational,maddison2017filtering,le2017auto}. A particle filter (PF) propagates $N$ weighted particles $(\ton{w}, \ton{x})$ to approximate the filtering distribution $p_\theta(x_t|y_{1:t})$. It first specifies a proposal distribution $r(x_t|x_{t-1},y_t).$ Then it iteratively samples new particles from the proposal, followed by a reweighing procedure: 
\begin{equation}\label{eq:PF_update}
 \tnf{x} \sim r(\cdot|\tmn{ x},y_t), \quad\quad \tn{\widehat w}=\frac{p_\theta(\tnf{x}|\tmn{ x}) \Nc(y_t; H \tnf{x}, R) }{r(\tnf{x}|\tmn{ x}, y_t)} \tmn{w}.
\end{equation}
A resampling step is performed when necessary, e.g., if the effective sample size drops below a certain threshold:
\begin{equation}\label{eq:PF_resample}
(\tn{w}, \tn{x}) = \Bigl(\frac{1}{N}, \widehat x_t^{\tn{a}}\Bigr), \quad\quad \text{where } \tn{a} \sim \text{Categorical}\bigl(\cdot|\ton{\widehat w}\bigr).
\end{equation}
We refer the reader to \cite{doucet2009tutorial} for a more general introduction to PF. Here we consider PF with the optimal proposal $r(\tnf{x}|\tmn{x}, y_t) \triangleq p_\theta(\tnf{x}|\tmn{x}, y_t)$, as it is readily available for the family of SSMs that is considered in this paper \cite{doucet2000sequential,sanzstuarttaeb}. The unnormalized weights $\tn{\widehat w}$ also provides an approximation to the data log-likelihood. Define
\begin{equation}
    \LhPF(\theta) = \sum_{t=1}^T \log \Big( \sum_{n=1}^N \tn{\widehat w} \Big).
\end{equation}
Note that the weights $\tn{\widehat w}$ depends implicitly on $\theta$. It is a well-known result \cite{del2004feynman,andrieu2010particle} that the data likelihood $\exp \L(\theta)$ can be unbiasedly estimated, and a lower bound for $\L(\theta)$ can be derived:
\begin{equation}\label{eq:PFlike}
\Expect[\exp \LhPF (\theta)] = \exp \L(\theta), \quad\quad \Expect[\LhPF (\theta)] \le \L(\theta),
\end{equation}
where the last inequality follows from Jensen's inequality. We summarize the optimal PF algorithm and parameter learning procedure (AD-PF) in \cref{alg:PF} and \cref{alg:AD-PF} below. AD-PF with truncated backprop (AD-PF-T) can be built in a similar fashion as AD-EnKF-T (\cref{alg:AD-EnKF-T}), and is omitted. For a derivation of update rules of optimal PF, see e.g. \cite{sanzstuarttaeb}. For AD-PF where the proposal is learned from data, see e.g. \cite{naesseth2018variational,maddison2017filtering,le2017auto}. Following the same authors, the gradients from the resampling steps are discarded. Therefore, $\nt \LhPF$ is biased and inconsistent as $N\rightarrow\infty$ (see \cite{corenflos2021differentiable}).

\begin{algorithm}
\caption{Optimal Particle Filtering and Log-likelihood Estimation}
\label{alg:PF}
\begin{algorithmic}[1]
\Statex {\small {\bf Input}:  $\theta = \{\alpha,\beta\}, y_{1:T}, (w_0^{1:N}, x_0^{1:N}).$ (If $(w_0^{1:N}, x_0^{1:N})$ is not specified, $x_0^n \iidsim  p_0(x_0)$, $w_0^n=1/N$.)}
\State  {\bf Initialize} $\LhPF(\theta)=0.$ \label{eq:PF_init_alg}
\State Set $S=H Q_\beta H^\top + R$ and $K=Q_\beta H^\top S^{-1}.$
\For {$t= 1, \ldots, T$}
\State Set $\tnf{x} = (I-KH) F_\alpha(\tmn{x}) + K y_t + \tn{\xi}$, where $\tn{\xi}\sim \Nc(0, Q_\beta)$. \label{eq:PF_forecast_alg}
\State Set $\tnf{w} = \Nc\bigl(y_t; H F_\alpha(\tmn{x}), S\bigr)$
\State Set $(\tn{w}, \tn{x}) = \Bigl(\frac{1}{N}, \widehat x_t^{\tn{a}}\Bigr)$, where $\tn{a} \sim \text{Categorical}\bigl(\cdot|\ton{\widehat w}\bigr)$ \Comment{Resampling step}
\State    Set $\LhPF(\theta) \leftarrow \LhPF(\theta) + \log \Big( \sum_{n=1}^N \tnf{w} \Big).$ 
\EndFor
\Statex {\bf Output}: PF particles $(w_{0:T}^{1:N}, x_{0:T}^{1:N})$. Log-likelihood estimate $\LhPF(\theta)$.
\end{algorithmic}
\end{algorithm}

\begin{algorithm}
\caption{Auto-differentiable Particle Filtering (AD-PF)}
\label{alg:AD-PF}
\begin{algorithmic}[1]
\Statex {\bf Input}:  Observations $y_{1:T}$. Learning rate $\eta$.
\State {\bf Initialize} SSM parameter $\theta^0$ and set $k=0.$
\While {not converging}
\State $(w_{0:T}^{1:N}, x_{0:T}^{1:N}), \LhPF(\theta^k) = \textsc{ParticleFilter}(\theta^k, y_{1:T}).$ \label{eq:PF_main_1}
\Comment{\cref{alg:PF}}
\State Compute $\nt \LhPF(\theta^k)$ by auto-differentiating the map $\theta^k \mapsto \LPF(\theta^k).$
\State Set $\theta^{k+1} = \theta^{k} + \eta \nt \LhPF(\theta^k)$ and $k\leftarrow k+1.$
\EndWhile
\Statex {\bf Output}: Learned SSM parameter $\theta^k$  and PF particles $(w_{0:T}^{1:N},x_{0:T}^{1:N})$.
\end{algorithmic}
\end{algorithm}

\section{Expectation-Maximization with Ensemble Kalman Filters}\label{sec:em-appen}
Here we review the EM methods introduced in, e.g., \cite{pulido2018stochastic,bocquet2020bayesian}.
Consider the log-likelihood expression
\begin{align}
    \log p_\theta(y_{1:T}) &= \log \int p_\theta(y_{1:T}, x_{0:T}) \dd x_{0:T} \\
        &= \log \Expect_{q(x_{0:T}|y_{1:T})} \Big[ \frac{p_\theta(y_{1:T}, x_{0:T})}{q(x_{0:T}|y_{1:T})} \Big]\\
        &\ge \Expect_{q(x_{0:T}|y_{1:T})} \big[ \log p_\theta(y_{1:T}, x_{0:T}) - \log q(x_{0:T}|y_{1:T}) \big] := \L(q,\theta). \label{eq:EM-obj}
\end{align}
The algorithm maximizes $\L(q,\theta)$, which alternates between two steps:
\begin{enumerate}
    \item (E-step) Fix $\theta^k$. Find $q^k$ that maximizes $\L(q,\theta^k)$, which can be shown to be the exact posterior 
    \begin{equation*}
        q^k(x_{0:T}|y_{1:T}) = p_{\theta^k}(x_{0:T}|y_{1:T}).
    \end{equation*}
    \item (M-step) Fix $q^k$. Find $\theta^{k+1}$ that maximizes $\L(q^k,\theta)$:
    \begin{equation}\label{eq:EM-mstep}
        \theta^* = \argmax_\theta \Expect_{q^k(x_{0:T}|y_{1:T})} \big[ \log p_\theta(y_{1:T}, x_{0:T}) \big].
    \end{equation}
\end{enumerate}
The exact posterior in the E-step is not always available, so one can run an Ensemble Kalman Smoother (EnKS) to approximate the distribution $p_{\theta^k}(x_{0:T}|y_{1:T})$ by a group of equally-weighted particles $\oneton{x_{0:T}^{k,n}}$, and replace the expectation in the M-step by an Monte-Carlo average among these particles.

Consider again the nonlinear SSM defined by equations \eqref{eq:main_state}-\eqref{eq:main_obs}-\eqref{eq:main_init} with dynamics $F_\alpha$ and noise covariance $Q_\beta.$ The key to the methods proposed in, e.g., \cite{pulido2018stochastic,bocquet2020bayesian} is that, for the optimization problem \cref{eq:EM-mstep} in the M-steps, optimizing w.r.t.\ $\alpha$ is hard, but optimizing w.r.t.\ $\beta$ is easy. If we write $\oneton{x_{0:T}^{n}}$ as an approximation of $q(x_{0:T}|y_{1:T})$ (we drop $k$ here for notation convenience), the optimization problem \cref{eq:EM-mstep} can be rewritten as (using Monte-Carlo approximation)
\begin{align*}
    \argmax_\theta \frac{1}{N} \sum_{n=1}^N \log p_\theta(y_{1:T}, x_{0:T}^{n})
    &= \argmax_{\alpha, \beta} \frac{1}{N} \sum_{n=1}^N \sum_{t=1}^T \log \Nc \bigl(\tn{x}; F_\alpha (\tmn{ x}), Q_\beta\bigr).
\end{align*}
When $\alpha$ is given, optimizing w.r.t.\ $\beta$ becomes the usual MLE computation for covariance matrix of a multivariate Gaussian. We summarize the EM algorithm in Algorithm \ref{algorithm:EM}. Note that this special type of M-step update only applies to SSMs with transition kernel of the form \cref{eq:main_state}. For more general SSMs, updates for $\beta$ may not admit a closed form expression, and one may have to resort to a joint optimzation of $(\alpha,\beta)$ in the M-steps.

\begin{algorithm}
\setlength{\abovedisplayskip}{3pt}
\setlength{\belowdisplayskip}{3pt}

\caption{Expectation Maximization (EM) with Ensemble Kalman Filter}
\label{algorithm:EM}
\begin{algorithmic}[1]
\Statex {\bf Input}:  Observations $y_{1:T}$. Learning rate $\eta$. Smoothing parameter $S$. Inner-loop gradient ascent steps $J$.
\State {\bf Initialize} SSM parameter $\theta^0=(\alpha^0, \beta^0)$. Initialize $k=0.$
\While {not converging}
\State $x_{0:T}^{1:N} = \textsc{EnsembleKalmanFilter}(\theta^k, y_{1:T})$. \Comment{\cref{alg:EnKF}}
\State Perform smoothing for lag $S$ (see e.g. \cite{katzfuss2020ensemble}) and update $x_{0:T}^{1:N}$. 
\State Set $$Q_{\beta^{k+1}} = \frac{1}{NT} \sum_{n=1}^N \sum_{t=1}^T \Big( \tn{x} - F_{\alpha^{k}}(\tmn{ x}) \Big) \Big( \tn{x} - F_{\alpha^{k}}(\tmn{ x}) \Big)^{\top}.$$ 
\State Initialize $\alpha^{k,0}=\alpha^k.$
\For {$j=0,\dots,J-1$}
\State $\alpha^{k,j+1} = \alpha^{k,j} + \eta \na \LEMEnKF(\alpha^{k,j}, x_{0:T}^{1:N}), $ where $$\LEMEnKF(\alpha, x_{0:T}^{1:N}) \triangleq \frac{1}{N} \sum_{n=1}^N \sum_{t=1}^T \log \Nc(\tn{x}; F_\alpha (\tmn{ x}), Q_{\beta^{k+1}}).$$
\EndFor
\State Set $\alpha^{k+1} = \alpha^{k,J}$, $k\leftarrow k+1.$
\EndWhile
\Statex {\bf Output}: Learned SSM parameter $\theta^k=(\alpha^k, \beta^k)$  and particles $x_{0:T}^{1:N}.$
\end{algorithmic}
\end{algorithm}

\section{Implementation Details and Additional Performance Metrics}\label{sec:implement-appen}
\subsection{Linear-Gaussian Model}
The relative $L^2$ errors of the log-likelihood and gradient estimates in \cref{sec:linear_converge} are given by 
\begin{equation}
    \frac{\Expect[(\LhEnKF(\theta)-\L(\theta))^2]^{1/2}}{|\L(\theta)|},  
\end{equation}
and 
\begin{equation}
    \frac{\Expect[|\na \LhEnKF(\theta)-\na\L(\theta)|^2]^{1/2}}{|\na \L(\theta)|}, \quad\quad \frac{\Expect[|\nb \LhEnKF(\theta)-\nb\L(\theta)|^2]^{1/2}}{|\nb \L(\theta)|}.
\end{equation}

The tapering matrix $\rho$ in \cref{sec:linear_loc} is defined through
$[\rho]_{i,j} = \varphi \Big(\frac{|i-j|}{r}\Big),$ where
\begin{equation}\label{eq:tapering}
    \varphi(z)= \left. 
    \begin{cases}
    1-\frac{5}{3}z^2+\frac{5}{8}z^3+\frac{1}{2}z^4-\frac{1}{4}z^5, & \text{if } 0 \le z \le 1, \\
    4-5z+\frac{5}{3}z^2+\frac{5}{8}z^3-\frac{1}{2}z^4+\frac{1}{12}z^5-\frac{2}{3z}, & \text{if } 1 \le z \le 2,  \\
    0, & \text{if } z \ge 2.
    \end{cases}\right.
\end{equation}

We initialize the parameter learning at $\alpha^0=(0.5, 0.5, 0.5)$ and $\beta^0=(1,0.1)$.
For a fair comparison of convergence performance, we set the learning rate to be $\eta_\alpha=10^{-4}$ for $\alpha$ and $\eta_\beta=10^{-3}$ for $\beta$ for all methods, and perform gradient ascent for 1000 iterations.

\subsection{Lorenz-96}
\subsubsection{Additional Performance Metrics}
\paragraph{Averaged Diagnosed Error Level:} Measures the level of the learned model error:
\begin{equation}\label{eq:diagnosed-error}
    \sigma_\beta = \sqrt{\text{Tr}(Q_\beta)/{d_x}}.
\end{equation}

\paragraph{Forecast RMSE (RMSE-f):} Measures the forecast error of the learned vector field $f_\alpha$ compared to the reference vector field $f^*$ . We select $P=4000$ points $\{ x_p \}_{p=1}^P$ uniformly at random from the attractor of the reference system, by simulating a long independent run. The forecast RMSE is defined as
\begin{equation}\label{eq:RMSE-f}
\rmsef = \sqrt{ \frac{1}{d_x P} \sum_{p=1}^P \Big| F_{\alpha}(x_p) - F^*(x_p) \Big|^2} \,.
\end{equation}
Here we recall that $F^*$ and $F_\alpha$ are the $\Delta_s$-flow maps with fields $f^*$ and $f_\alpha,$ where $\Delta_s$ is the time between observations.
\paragraph{Analysis/Filter RMSE (RMSE-a):} Measures the state estimation error of the learned state transition kernel $p_\theta(x_t|x_{t-1})$. Let $(x_{0:T}^\train, y_{1:T}^\train)$ be one sequence of training data. Then a filtering run is performed over the data $y_{1:T}^\train$ with the learned state transition kernel $p_\theta(x_t|x_{t-1})$, using EnKF/PF, depending on which method is used for training. The (weighted) ensemble/particle mean $\bar x_{1:T}$ is recorded and the filter/analysis RMSE is defined as
\begin{equation}\label{eq:RMSE-a}
\rmsea = \sqrt{\frac{1}{d_x (T-T_b)}\sum_{t=T_b}^{T} \big| \bar x_t - x_t^\train \big|^2}.
\end{equation}
Here $T_b$ is the number of burn-in steps. For simplicity, we set $T_b = \lfloor T/5 \rfloor$. If multiple sequences of training data are used, then the quantity inside the square root is further averaged among all training sequences. 
\paragraph{Test Log-likelihood:} Measures the approximate data log-likelihood under the learned state transition kernel. Let $(x_{0:T}^\test,y_{1:T}^\test)$ be a sequence of test data drawn from the reference model independent of the training data used to estimate $\theta$.
The test log-likelihood is defined as $\LhEnKF(\theta)$ (or $\LhPF(\theta)$) computed over $y_{1:T}^\test$ using the corresponding EnKF/PF approximation, depending on which method is used for training.

\subsubsection{Implementation Details}
We use Adam \cite{kingma2014adam} throughout this experiment, with a learning rate that decays polynomially with the number of training iterations. Specifically, let $i$ be the index of current iteration. The learning rate $\eta_i$ is defined as:
\begin{equation}
\eta_i = \left.
\begin{cases}
\eta_0,  & i \le I_0,\\
\eta_0 (i-I_0)^{-\tau}, & i > I_0,\\
\end{cases}
\right.
\end{equation}
where $\eta_0$, $I_0$, and $\tau$ are hyperparaters to be specified. We set $I_0=10$ throughtout the experiment. $\beta$ is initialized at $\beta^0=2\cdot \bm{1}$ where $\bm{1}\in\R^{d_x}$ is the all-ones vector.

\paragraph{Parameterized Dynamics}
$\alpha$ is initialized at $\alpha^0=\bm{0}\in \R^{18}$. For AD-EnKF-T and AD-PF-T we set $\eta_0=10^{-1}$ and $\tau=0.5$. For EM we set $\eta_0=10^{-1}$, $\tau=1$, inner-loop gradient ascent steps $J=3$ and smoothing parameter $S=0$. The choice of $S$ is selected from $\{0,2,4,6,8\}$, following \cite{bocquet2020bayesian,brajard2020combining}. Although EM-based approaches in theory require to compute the smoothing rather than filtering distribution (see e.g. \cite{bishop2006pattern,bocquet2020bayesian}), in practice the ensemble-based smoothing algorithms may fail to produce an accurate approximation. All hyperparameters are tuned on validation sets. We also find that only the choice of $\eta_0$ is crucial to the convergence of AD-EnKF-T empirically, while other hyperparameters ($I_0$ and $\tau$) have small effects on the convergence speed.

\paragraph{Fully Unknown Dynamics}
NN weights $\alpha^0$ are initialized at random using PyTorch's default initialization. For AD-EnKF-T, AD-PF-T and EM we set $\eta_0=10^{-2}$ and $\tau=1$. For EM we set inner-loop gradient ascent steps $J=3$ and smoothing parameter $S=0$. The structure of $f_\alpha^{\NN}$ is detailed in \cref{fig:nn-structure}.

\paragraph{Model Correction}
NN weights $\alpha^0$ are initialized at random using PyTorch's default initialization. For AD-EnKF-T and AD-PF-T we set $\eta_0=10^{-3}$ and $\tau=0.75$. For EM we set $\eta_0=10^{-3}$ and $\tau=1$, $J=3$ and $S=0$. The structure of $g_\alpha^{\NN}$ is the same as $f_\alpha^{\NN}$ in the previous experiment, see \cref{fig:nn-structure}.

\begin{figure}[!htb]
\begin{subfigure}[c]{.48\linewidth}
\centering
\begin{tikzpicture}[scale = 0.8, every node/.style={scale=0.8}]
\tikzstyle{main}=[circle, minimum size = 8mm, thick, draw =black!80, node distance = 10mm]
\tikzstyle{connect}=[-latex, thick]
\tikzstyle{rconnect}=[-latex, thick, draw=blue]
\tikzstyle{box}=[rectangle, draw=black!100]
    \node[main] (in) {$x$};
    \node [draw,rectangle split, rectangle split horizontal, rectangle split parts=3,minimum height=7mm] (CNN1) [below=of in]
        {$\text{CNN}_1^{(1)}$
         \nodepart{two} $\text{CNN}_1^{(2)}$
         \nodepart{three} $\text{CNN}_1^{(3)}$};
    \node[main, below=0.4in of CNN1.two split,scale=0.8] (times) {\Large$\times$};
    \node [box,minimum height=7mm] (CNN2) [below=0.8in of CNN1]{$\text{CNN}_2$};
    \node [box,minimum height=7mm] (CNN3) [below=of CNN2]{$\text{CNN}_3$};
    \node[main, below=of CNN3,scale=0.8] (out) {$f_\alpha^{\NN}(x)$};
    \path (in) edge [connect] (CNN1);
  \path (CNN1.two south) edge [connect] (times);
  \path (CNN1.three south) edge [connect] (times);
  \path (CNN1.one south) edge [connect] (CNN2);
  \path (times) edge [connect] (CNN2);
  \path (CNN2) edge [connect] (CNN3);
  \path (CNN3) edge [connect] (out);
\end{tikzpicture}
\subcaption{Network structure}
\end{subfigure}
{\small
\begin{subtable}{0.48\textwidth}
  \centering
   \begin{tabular}[C]{|l|l|l|l|}
    \hline
         & \#in & \#out & kernel size \\ \hline
    $\text{CNN}_1$ & 1 & 72(24$\times$3) & 5, circular \\ \hline
    $\text{CNN}_2$ & 48(24$\times$2) & 37  & 5, circular \\ \hline
    $\text{CNN}_3$ & 37 & 1 & 1 \\ \hline
    \end{tabular}
  \subcaption{Network details}
  \end{subtable}
  }
\caption{Structure of $f_\alpha^{\NN}$. Output channels of $\text{CNN}_1$ is divided into three groups of equal length $\text{CNN}_1^{(1)}$, $\text{CNN}_1^{(2)}$ and $\text{CNN}_1^{(3)}$. Input channels to $\text{CNN}_2$ is a concatenation of $\text{CNN}_1^{(1)}$ and $(\text{CNN}_1^{(2)} \times \text{CNN}_1^{(3)})$, where the multiplication is point-wise.}
\label{fig:nn-structure}
\end{figure}
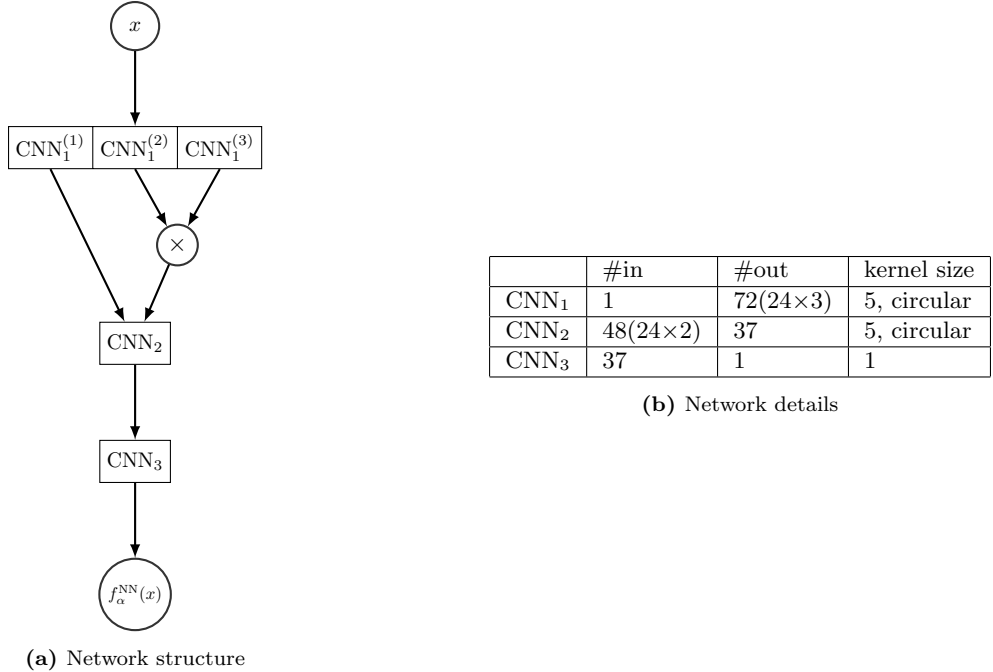

\end{document}